\documentclass[11pt]{article}

\usepackage[round]{natbib}
\usepackage{graphicx}

\usepackage{hyperref}

\usepackage{amstext,latexsym,amsbsy,amsfonts,amsmath,amssymb,amsthm}
\usepackage{algorithm}
\usepackage{algorithmic}
\usepackage{subfig}
\usepackage{booktabs}
\usepackage{xcolor}
\usepackage{bbm}
\usepackage{diagbox}
\usepackage{authblk}
\usepackage{arxiv}
\usepackage{nicefrac} 
\usepackage{multirow}

\long\def\acks#1{\vskip 0.3in\noindent{\large\bf Acknowledgments}\vskip 0.2in
\noindent #1}

\newcommand{\trace}{\ensuremath{tr}}
\long\def\comment#1{}
\renewcommand\vec[1]{\ensuremath\boldsymbol{#1}}

\renewcommand{\trace}{\ensuremath{\operatorname{tr}}}
\newcommand{\Ecal}{\ensuremath{\mathcal{E}}}
\newcommand{\Ocal}{\ensuremath{\mathcal{O}}}

\newcommand{\Hcal}{\ensuremath{\mathcal{H}}}
\newcommand{\Hbold}{\ensuremath{\boldsymbol{H}}}

\newtheorem{innercustomthm}{Theorem}
\newenvironment{customthm}[1]
  {\renewcommand\theinnercustomthm{#1}\innercustomthm}
  {\endinnercustomthm}
  
  \newtheorem{innercustomlem}{Lemma}
\newenvironment{customlem}[1]
  {\renewcommand\theinnercustomlem{#1}\innercustomlem}
  {\endinnercustomlem}

  \newtheorem{innercustomprop}{Proposition}
\newenvironment{customprop}[1]
  {\renewcommand\theinnercustomprop{#1}\innercustomprop}
  {\endinnercustomprop}
\theoremstyle{plain}

\newtheorem{thm}{Theorem}
\numberwithin{thm}{section}

\usepackage{eqparbox}

\usepackage{etoolbox}  
\makeatletter
\patchcmd{\algorithmic}{\addtolength{\ALC@tlm}{\leftmargin} }{\addtolength{\ALC@tlm}{\leftmargin}}{}{}
\makeatother

\providecommand{\keywords}[1]{\textbf{\textit{Index terms---}} #1}

\title{On Efficient Multilevel Clustering via Wasserstein Distances}

\makeatletter
\newcommand\email[2][]%
   {\newaffiltrue\let\AB@blk@and\AB@pand
      \if\relax#1\relax\def\AB@note{\AB@thenote}\else\def\AB@note{\relax}%
        \setcounter{Maxaffil}{0}\fi
      \begingroup
        \let\protect\@unexpandable@protect
        \def\thanks{\protect\thanks}\def\footnote{\protect\footnote}%
        \@temptokena=\expandafter{\AB@authors}%
        {\def\\{\protect\\\protect\Affilfont}\xdef\AB@temp{#2}}%
         \xdef\AB@authors{\the\@temptokena\AB@las\AB@au@str
         \protect\\[\affilsep]\protect\Affilfont\AB@temp}%
         \gdef\AB@las{}\gdef\AB@au@str{}%
        {\def\\{, \ignorespaces}\xdef\AB@temp{#2}}%
        \@temptokena=\expandafter{\AB@affillist}%
        \xdef\AB@affillist{\the\@temptokena \AB@affilsep
          \AB@affilnote{}\protect\Affilfont\AB@temp}%
      \endgroup
       \let\AB@affilsep\AB@affilsepx
}
\makeatother

\author[1]{Viet Huynh}
\author[2]{Nhat Ho}
\author[1]{Nhan Dam}
\author[3]{XuanLong Nguyen}
\author[4]{Mikhail Yurochkin}
\author[5]{Hung Bui}
\author[1,5]{Dinh Phung}

\affil[1]{Faculty of Information Technology, Monash University}
\email{\url{{viet.huynh, nhan.dam, dinh.phung}@monash.edu}}
\affil[2]{Department of EECS,University of California, Berkeley}
\email{\url{minhnhat@berkeley.edu}}
\affil[3]{Department of Statistics, University of Michigan}
\email{\url{xuanlong@umich.edu}}
\affil[4]{IBM Research, USA}
\email{\url{mikhail.yurochkin@ibm.com}}
\affil[5]{VinAI Research, Hanoi, Viet Nam}
\email{\url{v.hungbh1@vinai.io}}

\date{}

\begin{document}

\maketitle

\begin{abstract}
We propose a novel approach to the problem of multilevel clustering, which aims to simultaneously partition data in each group and discover grouping patterns among groups in a potentially large hierarchically structured corpus of data. Our method involves a joint optimization formulation over several spaces of discrete probability measures, which are endowed with Wasserstein distance metrics. We propose several variants of this problem, which admit fast optimization algorithms, by exploiting the connection to the problem of finding Wasserstein barycenters.  Consistency properties are established for the estimates of both local and global clusters. Finally, experimental results with both synthetic and real data are presented to demonstrate the flexibility and scalability of the proposed approach.
\end{abstract}

\section{Introduction} \label{Section:introduction}
In numerous applications in engineering and sciences, data are often organized in a multilevel structure. For instance, a typical structural view of text data in machine learning is to have words grouped into documents and documents grouped into corpora. A prominent strand of modeling and algorithmic work in the past couple of decades has been to discover latent multilevel structures from these hierarchically structured data. For specific clustering tasks, one may be interested in simultaneously partitioning the data in each group (to obtain local clusters) and partitioning a collection of data groups (to obtain global clusters). 
A concrete example is the problem of clustering images (i.e. global clusters) where each image  contains multiple annotated regions (i.e. local clusters) \citep{Oliva-2001}. While hierarchical clustering techniques may be employed to find a tree-structured clustering given a collection of data points, they are not applicable to discovering the \textit{nested} structure of multilevel data. Bayesian hierarchical models provide a powerful approach, exemplified by influential work such as \citep{Blei-etal-03,Pritchard-etal-00,Teh-etal-06}. More specific to the simultaneous and multilevel clustering problem, we mention the paper of \cite{Rodriguez-etal-08}. In this interesting work, a Bayesian nonparametric model called the nested Dirichlet process (NDP) model was introduced that enables the inference of clustering of a collection of probability distributions from which different groups of data are drawn. With suitable extensions, this modeling framework has been further developed for simultaneous multilevel clustering, see, for instance, \citep{Wulsin-2016,Vu-2014,Viet-2016}. 
The focus of this paper is on the multilevel clustering problem motivated in the aforementioned modeling papers, \textit{but we shall take a pure optimization approach}. This paper includes substantially new results compared to our preliminary conference version~\citep{Ho_2017_Multilevel}. We aim to formulate optimization problems that enable the discovery of multilevel clustering structures hidden in grouped data. Our technical approach is inspired by the role of optimal transport distances in hierarchical modeling and clustering problems. The optimal transport distances, also known as Wasserstein distances \citep{Villani-03}, have been shown to be the natural distance metric for the convergence theory of latent mixing measures arising in both mixture models \citep{Nguyen-13} and hierarchical models \citep{Nguyen-2016}. They are also intimately connected to the problem of clustering --- this relationship goes back at least to the work of \citep{Pollard-1982}, where it is pointed out that the well-known K-means clustering algorithm can be directly linked to the quantization problem --- the problem of determining an optimal finite discrete probability measure that minimizes its second-order Wasserstein distance from the empirical distribution of given data \citep{Graf-2000}.

If one is to perform simultaneous K-means clustering on hierarchically grouped data, both at the global level (among groups), and local level (within each group), then this can be achieved by a joint optimization problem defined with suitable notions of Wasserstein distances inserted into the objective function. In particular, multilevel clustering requires to optimize in the space of probability measures defined in different levels of abstraction, including the space of measures of measures on the space of grouped data. Our goal, therefore, is to formulate this optimization precisely, to develop algorithms for solving the optimization problem efficiently, and to make sense of the obtained solutions in terms of statistical consistency. 

The algorithms that we propose address directly a multilevel clustering problem formulated from a pure optimization viewpoint, but they may also be taken as a fast approximation to the inference of latent mixing measures that arises in the nested Dirichlet process in \citep{Rodriguez-etal-08}. From a statistical viewpoint, we shall establish a consistency theory for our multilevel clustering problem in the manner achieved for K-means clustering \citep{Pollard-1982}. From a computational viewpoint, quite interestingly, we will be able to explicate and exploit the connection between our optimization formulation and the problem of finding the Wasserstein barycenter \citep{Carlier-2011}, a computational problem that has also attracted much recent interest, e.g., \citep{Cuturi-2014, Lin_2020}.

In summary, the main contributions offered in this work include: \textit{(i)} several new optimization formulations of the multilevel clustering problem using Wasserstein distances defined on different levels of the hierarchical data structure; \textit{(ii)} fast algorithms by exploiting the connection of our formulation to the Wasserstein barycenter problem; \textit{(iii)} consistency theorems established for the proposed estimation under a very mild condition of data distributions;  \textit{(iv)} several flexible alternatives by introducing constraints that encourage the borrowing of strength among local and global clusters; \textit{(v)} finally, demonstration of efficiency and flexibility of our approach in a number of simulated and real datasets.  

The paper is organized as follows. Section \ref{Section:prelim} provides the preliminary background on Wasserstein distances, Wasserstein barycenter, and the connection between  K-means clustering and the quantization problem. Section \ref{Section:multilevel_Wasserstein} presents several optimization formulations of the multilevel clustering problem and the algorithms for solving them. Sections \ref{Section:multilevel_Wasserstein_context} and \ref{Section:robust_mutilevel_Wasserstein_median} present the alternatives of our proposed formulations under two scenarios: multilevel structure data with context and first order Wasserstein distance replacing second order Wasserstein distance for robust clustering. Section \ref{Section:consistency_multilevel_Kmeans} establishes the consistency results for the estimators introduced in previous sections. Section \ref{Section:data_analysis} presents
empirical studies with both synthetic and real datasets. Finally, we conclude the paper in Section \ref{Section:discussion}. Additional technical details, including all proofs, are given in the appendices.

\section{Background} \label{Section:prelim}

For any given subset $\Theta \subset \mathbb{R}^{d}$, let $
\mathcal{P}(\Theta)$ denote the space of Borel probability measures on $\Theta$.  The Wasserstein space of order $r \in [1,\infty)$ of probability measures on $\Theta$ is defined as $\mathcal{P}_{r}(\Theta)=\biggr\{G \in \mathcal{P}(\Theta): \ \int \limits {\|x\|^{r}}\mathrm{d}G(x)<\infty \biggr\}$,
where $\|.\|$ denotes the Euclidean metric in $\mathbb{R}^{d}$. Additionally, for any $k \geq 
1$ the probability simplex is denoted by $\Delta_{k}=\left\{u \in \mathbb{R}^{k}: \ u_{i} \geq 0, \ 
\sum \limits_{i=1}^{k}{u_{i}}=1 \right\}$. Finally, let $\mathcal{O}_{k}(\Theta)$ (resp., $\mathcal{E}_{k}(\Theta)$) 
be the set of probability measures with at most (resp., exactly) $k$ support points in $\Theta$.

\paragraph{Wasserstein distances:}
For any elements $G$ and $G'$ in $\mathcal{P}_{r}(\Theta)$ where $r \geq 1$, the Wasserstein distance of order $r$ between $G$ and $G'$ is defined as (cf. \citep{Villani-03}):
\begin{eqnarray}
W_{r}(G,G')=\biggr(\mathop {\inf }\limits_{\pi \in \Pi(G,G')}{\int \limits_{\Theta^{2}}{\|x-y\|^{r}}\mathrm{d}\pi(x,y)}\biggr)^{1/r}, \nonumber
\end{eqnarray}
where $\Pi(G,G')$ is the set of all probability measures on $\Theta \times \Theta$ that have marginals $G$ and $G'$.  In words, $W_r(G,G')$ is the optimal cost of moving mass from $G$ to $G'$, where the cost of moving unit mass is proportional to $r$-power of Euclidean distance in $\Theta$. 
 
By the recursion of concepts, we can speak of measures of measures, and define a suitable distance metric on this abstract space:  the space of Borel measures on $\mathcal{P}_{r}(\Theta)$, to be denoted by $\mathcal{P}_{r}(\mathcal{P}_{r}(\Theta))$. This is also a Polish space (i.e. complete and separable metric space) as $
\mathcal{P}_{r}(\Theta)$ is a Polish space. It will be endowed with a Wasserstein metric of  order $r$ that is induced by a metric $W_{r}$ on $\mathcal{P}_{r}(\Theta)$ as follows (cf. Section 3 of \citep{Nguyen-2016}): for any $\mathcal{D},\mathcal{D'} \in \mathcal{P}_r(\mathcal{P}_r(\Theta))$
\begin{eqnarray}
W_{r}(\mathcal{D},\mathcal{D}'):=\biggr(\mathop {\inf }\limits_{\pi \in \Pi(D,D')}{\int \limits_{\mathcal{P}_{r}(\Theta)^{2}}{W_{r}^{r}(G,G')}\mathrm{d}\pi(G,G')}\biggr)^{1/r}, \nonumber
\end{eqnarray}where $\Pi(\mathcal{D},\mathcal{D}')$ is the set of all probability measures on $
\mathcal{P}_{r}(\Theta) \times \mathcal{P}_{r}(\Theta)$ that have marginals $\mathcal{D}
$ and $\mathcal{D}'$. In words, $W_r(\mathcal{D},\mathcal{D'})$ corresponds to the optimal cost of moving mass from $\mathcal{D}$ to $\mathcal{D'}$, where the cost of moving unit mass in its space of support $\mathcal{P}_r(\Theta)$ 
is proportional to the $r$-power of the $W_r$ distance in $\mathcal{P}_r(\Theta)$.
Note a slight abuse of notation --- $W_r$ is used for both
$\mathcal{P}_r(\Theta)$ and $\mathcal{P}_r(\mathcal{P}_r(\Theta))$, but it should be clear which one is being used from context.
\paragraph{Entropic version of Wasserstein distances:} The Wasserstein distance has been shown to have expensive computational complexity~\citep{Pele-2009-Fast} when the probability measures are discrete. To account for the computational complexity,~\cite{Cuturi-2013} proposed the entropic version of Wasserstein distances, which is given by:
\begin{align}
		\hat W_{r}^{r}(G, G') : = \inf \limits_{\pi \in \Pi(G,G')}{\int \limits_{\Theta^{2}}{\|x-y\|^{r}}\mathrm{d}\pi(x,y)} + \tau \text{H}(\pi|G \otimes G'), \label{eq:entropic_Wasserstein}
\end{align}
where $\tau > 0$ denotes the regularized parameter, $G \otimes G'$ denotes the product measure between $G$ and $G'$, and $\text{H}(\pi|G \otimes G')$ denotes the relative entropy between $\pi$ and $G \otimes G'$:
\begin{align*}
    \text{H}(\pi|G \otimes G') : = \int \log \Big(\frac{d\pi}{d G \otimes G'}(x,y)\Big) d\pi(x,y),
\end{align*}
where $\nicefrac{d\pi}{d G \otimes G'}(x,y)$ denotes the density of $\pi$ with respect to $G \otimes G'$.

When $G$ and $G'$ are discrete measures with at most $k$ supports, we can compute the entropic version of Wasserstein distances via the Sinkhorn algorithm. Furthermore, by choosing the regularized parameter $\tau$ at the order of $\varepsilon/ \log k$ where $\varepsilon$ stands for the desired error, the computational complexity of the Sinkhorn algorithm for approximating the Wasserstein distance is of the order $k^2/ \varepsilon^2$~\citep{Dvurechensky-2018-Computational}. The recent works of~\cite{Lin-2019-Efficient, Lin-2019-Acceleration} proposed the acceleration of the Sinkhorn algorithm with an improved complexity $k^{7/3}/ \varepsilon^{4/3}$ in terms of $\varepsilon$. Due to the favorable performance of the Sinkhorn algorithm, throughout the paper we utilize that algorithm to compute the entropic regularized version of Wasserstein distance between the probability measures. 
\paragraph{Wasserstein barycenter:}
Next, we present a brief overview of the Wasserstein barycenter problem, first studied in \citep{Carlier-2011} and subsequentially many others (e.g. \citep{Benamou-15, Solomon-15, Alvarez-16}). 
Given probability measures 
$P_{1}, P_{2}, \ldots, P_{N} \in \mathcal{P}_{2}(\Theta)$ for $N \geq 1$, their 
Wasserstein barycenter $\overline{P}_{N,\lambda}$ is such that
\vspace{-6pt}
\begin{eqnarray}
\overline{P}_{N,\lambda}=\mathop {\arg \min}\limits_{P \in \mathcal{P}_{2}(\Theta)}{\sum \limits_{i=1}^{N}{\lambda_{i}W_{2}^{2}(P,P_{i})}}, \label{eqn:Wasserstein_barycenter}
\end{eqnarray} 
where $\lambda \in \Delta_{N}$ denotes weights associated with $P_{1},\ldots,P_{N}$. When $P_{1},\ldots, P_{N}$ are discrete measures with finite number of atoms and the  weights $\lambda$ are uniform, it was shown in \citep{Anderes-2015} that the problem of finding Wasserstein barycenter $\overline{P}
_{N,\lambda}$ over the space $\mathcal{P}_{2}(\Theta)$ in equation~\eqref{eqn:Wasserstein_barycenter} is reduced to the search over only a much simpler space  $\mathcal{O}_{l}(\Theta)$ 
where $l=\sum \limits_{i=1}
^{N}{s_{i}-N+1}$ and $s_{i}$ is the number of components of $P_{i}$ for all $1 \leq i \leq 
N$.  Efficient algorithms for finding local solutions of the Wasserstein barycenter problem over $\mathcal{O}_{k}(\Theta)$ for some $k \geq 1$ have been studied recently in \citep{Cuturi-2014}. These algorithms will prove to be a useful building block for our method as we shall describe in the sequel.

\paragraph{K-means as quantization problem:}
The well-known $K$-means clustering algorithm can be viewed as solving an optimization problem that arises in the problem of quantization, a simple yet very useful connection \citep{Pollard-1982, Graf-2000} as follows. Given $n$ unlabeled samples $Y_{1},\ldots,Y_{n} \in \Theta$, we assume that these data are associated with at most $k$ clusters where $k \geq 1$ is some given number. The $K$-means problem finds the set $S$  containing at most $k$ elements $\theta_{1},\ldots, \theta_{k} \in \Theta$ that satisfies the following objective
\begin{eqnarray}
\mathop {\inf }\limits_{S : |S| \leq k}{\dfrac{1}{n}\sum \limits_{i=1}^{n}{d^{2}(Y_{i},S)}}, \label{eqn:original_Kmeans}
\end{eqnarray} where $d^{2}(Y_{i},S)=\min_{\theta\in S}\|Y_i-\theta\|^{2}$ is the square Euclidean distance from sample $Y_i$ to set $S$. Let $P_{n}=\dfrac{1}{n}\sum \limits_{i=1}^{n}{\delta_{Y_{i}}}$ be the empirical measure of data $Y_{1},\ldots,Y_{n}$ where $\delta_{Y}$ denotes the Dirac measure centred on Y. Then, problem \eqref{eqn:original_Kmeans} is equivalent to finding a discrete probability measure $G$ which has finite number of support points and solves:
\begin{eqnarray}
\mathop {\inf }\limits_{G \in \mathcal{O}_{k}(\Theta)}{W_{2}^{2}(G,P_{n})}. \label{eqn:Wasserstein_K_means}
\end{eqnarray} 
Due to the inclusion of the Wasserstein metric in its formulation, we call this a \emph{Wasserstein means problem}. This problem can be further thought of as a Wasserstein barycenter problem where $N=1$. In light of this observation, as noted in \citep{Cuturi-2014}, the algorithm for finding the Wasserstein barycenter offers an alternative for the popular Loyd's algorithm for determining the local minimum of the K-means objective. 

\section{Clustering with multilevel structure data} \label{Section:multilevel_Wasserstein}
Given $m$ groups of $n_{j}$ exchangeable data points $X_{j,i}$ where $1 
\leq j \leq m, 1 \leq i \leq n_{j}$, i.e. data are represented in a two-level grouping structure,  our goal is to learn about the two-level clustering structure of the data. We want to obtain simultaneously local clusters for each data group, and global clusters among all groups. 

\subsection{Multilevel Wasserstein means (MWM) algorithm} \label{Section:multilevel_kmeans}
For any $j=1,\ldots, m$, we denote the empirical measure for group $j$ by $P_{n_{j}}
^{j}:=\dfrac{1}{n_{j}}\sum \limits_{i=1}^{n_{j}}{\delta_{X_{j,i}}}$. Throughout this section, for simplicity of exposition, we assume that the numbers of both local and global clusters are either known or bounded above by a given number. In 
particular, for local clustering we allow group $j$ to have at most $k_{j}$ clusters for $j=1,\ldots, m$. For global clustering, we assume to have $M$ group (Wasserstein) means among the $m$ given groups. We now describe the high level idea of our proposed model, later elaborate its formal formulation, and demonstrate the connection to Bayesian hierarchical models. 

\paragraph{High level idea:}
For local clustering, for each $j = 1,\ldots, m$, performing a K-means clustering for group $j$, as expressed by \eqref{eqn:Wasserstein_K_means}, can be viewed as finding a finite discrete measure $G_{j} 
\in \mathcal{O}_{k_{j}}(\Theta)$ that minimizes squared Wasserstein distance $W_{2}^{2}
(G_{j},P_{n_{j}}^{j})$. For global clustering, we are interested in 
obtaining clusters out of $m$ groups, each of which is now represented by the discrete measure $G_j$, for $j=1,\ldots,m$.  Adopting again the viewpoint of \eqref{eqn:Wasserstein_K_means}, provided that all of $G_{j}$s are given, we can apply $K$-means quantization method to find their distributional clusters. The global clustering in the space of measures of measures on $\Theta$ can be succinctly expressed by\begin{eqnarray}
\mathop {\inf }\limits_{\mathcal{H} \in \mathcal{E}_{M}(\mathcal{P}_{2}(\Theta))}{W_{2}^{2}\biggr(
\mathcal{H},\dfrac{1}{m}\sum \limits_{j=1}^{m}{\delta_{G_{j}}}\biggr)}. \nonumber
\end{eqnarray}
However, $G_{j}$s are not known --- they have to be optimized through local clustering in each data group.
\paragraph{MWM problem formulation:} We have arrived at an objective function for jointly optimizing over both local and global clusters
\begin{eqnarray}
\mathop {\inf }\limits_{\substack {G_{j} \in \mathcal{O}_{k_{j}}(\Theta), \\\mathcal{H} \in 
\mathcal{E}_{M}(\mathcal{P}_{2}(\Theta))}}{\mathop {\sum }\limits_{j=1}^{m}{W_{2}
^{2}(G_{j},P_{n_{j}}^{j})}}
+ \lambda W_{2}^{2}(\mathcal{H},\dfrac{1}{m}\mathop {\sum }\limits_{j=1}^{m}{\delta_{G_{j}}}), \label{eqn:multilevel_Kmeans_typeone}
\end{eqnarray}
where $\lambda$ is a positive number used to balance the accumulative losses between the local and global clustering. We call the above optimization the problem of \emph{Multilevel Wasserstein Means (MWM)}.  The notable feature of MWM is that its loss function consists of two types of distances  associated with the hierarchical data structure: one is the distance in the space of measures, i.e. $W_{2}^{2}(G_{j},P_{n_{j}}^{j})$,  and the other in the space of measures of measures, i.e. $W_{2}^{2}(\mathcal{H},\dfrac{1}{m}\mathop {\sum }\limits_{j=1}^{m}{\delta_{G_{j}}})$. The global clustering in the space of measures of measures is also a special case of D2 clustering objective function in~\citep{Li_2006}. Our main difference with~\citep{Li_2006} is the joint optimization between the local clustering and global clustering to encourage the grouping structures in the multilevel Wasserstein means.

By adopting K-means optimization to both local and global clustering, the MWM problem might look formidable at first sight. 
Fortunately, it is possible to simplify this original formulation substantially,
by exploiting the structure of $\Hcal$.
Indeed, we can show that formulation \eqref{eqn:multilevel_Kmeans_typeone} is equivalent to the following optimization problem, which looks much simpler as it involves only measures on $\Theta$:\begin{eqnarray}
\mathop {\inf }\limits_{G_{j} \in \mathcal{O}_{k_{j}}(\Theta), \Hbold}
{\mathop {\sum }\limits_{j=1}^{m}{W_{2}^{2}(G_{j},P_{n_{j}}^{j})}+\dfrac{\lambda}{m} d_{W_{2}}^{2}(G_{j},\Hbold)}, \label{eqn:multilevel_K_means_typeone_first}
\end{eqnarray}
where $d_{W_{2}}^{2}(G,\Hbold) := \mathop {\min } \limits_{1 \leq i \leq M}{W_{2}^{2}(G,H_{i})}$ and $\Hbold=(H_{1},\ldots,H_{M})$, with each $H_i \in \mathcal{P}_2(\Theta)$. The proof of this equivalence is deferred to Proposition \ref{lemma:equivalence_multilevels_Kmeans} in Appendix~\ref{sec:appendix_b}. 

\paragraph{Connection between MWM and Bayesian hierarchical models:} The local measures obtained via Wasserstein means problems yield the local clustering of data in each group. Then, the K-means algorithm on these local measures leads to the global clustering of these groups. Therefore, the simultaneous local and global clusterings in the joint optimization formulation of MWM enable the discovery of nested multilevel structures hidden in grouped data. This property shares several similarities  to Bayesian hierarchical models, such as the nested Dirichlet process (NDP) \citep{Rodriguez-etal-08}.

Before going into the details of the algorithm for approximating the objective function \eqref{eqn:multilevel_K_means_typeone_first} in Section \ref{Section:mwm_algorithm}, we shall present some simpler cases, which help to illustrate some properties of the optimal solutions of equation~\eqref{eqn:multilevel_K_means_typeone_first}, while providing insights of subsequent developments of the MWM formulation.
\subsubsection{Properties of MWM in special cases} \label{Section:mwm_specical_cases}
\paragraph{Example 1.} Suppose $k_{j}=1$ and $n_{j}=n$ for all $1 \leq j \leq m$, and $M=1$. Write $\Hbold = H \in \mathcal{P}_2(\Theta)$. Under this setting, the objective function \eqref{eqn:multilevel_K_means_typeone_first} can be rewritten as
\begin{eqnarray}
\mathop {\inf }\limits_{\substack {\theta_{j} \in \Theta, \\ H \in \mathcal{P}_{2}(\Theta)}}{\sum \limits_{j=1}^{m}{\sum \limits_{i=1}^{n}{\|\theta_{j}-X_{j,i}\|^{2}}}}
+\lambda W_{2}^{2}(\delta_{\theta_{j}},H)/m, \label{eqn:special_case_multilevel_Kmeans_one} 
\end{eqnarray}where $G_{j}=\delta_{\theta_{j}}$ for any $1 \leq j \leq m$. From the result of Theorem \ref{theorem:upperbound_barycenter} in the Supplement, the Wasserstein barycenter of $G_j$ with uniform weight has exactly one component, i.e., 
\begin{eqnarray}
\mathop {\inf } \limits_{H \in \mathcal{P}_{2}(\Theta)}{\sum \limits_{j=1}^{m}{W_{2}^{2}(\delta_{\theta_{j}},H)}} & = & \mathop {\inf }\limits_{H \in \mathcal{E}_{1}(\Theta)}{\sum \limits_{j=1}^{m}{W_{2}^{2}(G_{j},H)}}
= \sum \limits_{j=1}^{m}{\|\theta_{j}-(\sum \limits_{i=1}^{m}{\theta_{i}})/m\|^{2}}, \nonumber
\end{eqnarray}
where the second infimum is achieved when $H=\delta_{(\sum \limits_{j=1}^{m}{\theta_{j}})/m}$. Thus, objective function \eqref{eqn:special_case_multilevel_Kmeans_one} may be rewritten as
\begin{eqnarray}
\mathop {\inf }\limits_{\theta_{j} \in \Theta}{\sum \limits_{j=1}^{m}{\sum \limits_{i=1}^{n}{\|\theta_{j}-X_{j,i}\|^{2}}}} +\| m\theta_{j}-(\sum \limits_{l=1}^{m}{\theta_{l}})\|^{2}/m^{3}. \nonumber
\end{eqnarray}Write $\overline{X}_{j}=(\sum \limits_{i=1}^{n}{X_{j,i}})/n$ for all $1 \leq j \leq m$.  As $m \geq 2$, we can check that the unique optimal solutions for the above optimization problem are $\theta_{j}=\biggr((m^2n+1)\overline{X}_{j}+\sum \limits_{i \neq j}
\overline{X}_{i}\biggr)/(m^{2}n+m)$ for any $1 \leq j \leq m$. If we further assume that our data $X_{j,i}$ are i.i.d samples from the probability measure $P^{j}$ having mean $\mu_{j}
=E_{X \sim P^{j}}(X)$ for any $1 \leq j \leq m$, the previous result implies that $\theta_{i} 
\not \to \theta_{j}$ for almost surely as long as $\mu_{i} \neq \mu_{j}$. As a 
consequence, if $\mu_{j}$'s are pairwise different, the multilevel Wasserstein means under that simple scenario of \eqref{eqn:multilevel_K_means_typeone_first} will not have identical centers among local groups. 

On the other hand, we have $W_{2}^{2}(G_{i},G_{j})=\|\theta_{i}-\theta_{j}\|^{2}=
\biggr(\dfrac{mn}{mn+1}\biggr)^{2}\|\overline{X}_{i}-\overline{X}_{j}\|^{2}$. Now, 
from the definition of Wasserstein distance 
\begin{eqnarray}
W_{2}^{2}(P_{n}^{i},P_{n}^{j}) & = & \mathop {\min }\limits_{\sigma}{\dfrac{1}{n}\sum \limits_{l=1}^{n}{\|X_{i,l}-X_{j,\sigma(l)}\|^{2}}} \geq \|\overline{X}_{i}-\overline{X}_{j}\|^{2}, \nonumber
\end{eqnarray} where $\sigma$ in the above sum varies over all the permutations of $\left\{1,2,\ldots,n
\right\}$ and the second inequality is due to Cauchy-Schwarz's inequality. It implies that as 
long as $W_{2}^{2}(P_{n}^{i},P_{n}^{j})$ is small, the optimal solution $G_{i}$ and $G_{j}
$ of \eqref{eqn:special_case_multilevel_Kmeans_one} will be sufficiently close to each other. By letting $n \to \infty$, we also achieve the same conclusion regarding the asymptotic behavior of $G_{i}$ and $G_{j}$ with respect to $W_2(P^{i},P^{j})$.
\paragraph{Example 2.} Let $k_{j}=1$ and $n_{j}=n$ for all $1 \leq  j \leq m$ and $M=2$. 
Write $\Hbold = (H_1,H_2)$. Moreover, assume that there is a strict subset A of $\left\{1,2,\ldots,m\right\}$ such that \begin{eqnarray}
& & C \cdot \mathop {\max }\biggr\{\mathop {\max }\limits_{i, j \in A}{W_{2}(P_{n}^{i},P_{n}^{j})},
\mathop {\max }\limits_{i, j \in A^{c}}{W_{2}(P_{n}^{i},P_{n}^{j})}\biggr\} \leq \mathop {\min }\limits_{i \in A, j \in A^{c}}{W_{2}(P_{n}^{i},P_{n}^{j})}, \nonumber 
\end{eqnarray}
where $C$ is some sufficiently large positive constant, i.e., the distance of empirical measures $P_{n}^{i}$ and $P_{n}^{j}$ when $i$ and $j$ belong to the same set $A$ or $A^{c}$ is much smaller than that when $i$ and $j$ do not belong to the same set. Under this condition, by using the argument from Example 1 we can write 
the objective function \eqref{eqn:multilevel_K_means_typeone_first} as 
\begin{eqnarray}
\mathop {\inf }\limits_{\substack {\theta_{j} \in \Theta, \\ H_{1} \in \mathcal{P}_{2}(\Theta)}}{\sum \limits_{j \in A}{\sum \limits_{i=1}^{n}{\|\theta_{j}-X_{j,i}\|^{2}}}+\dfrac{W_{2}^{2}(\delta_{\theta_{j}},H_{1})}{|A|}}+ \nonumber \\ 
\mathop {\inf }\limits_{\substack {\theta_{j} \in \Theta, \\ H_{2} \in \mathcal{P}_{2}(\Theta)}}{\sum \limits_{j \in A^{c}}{\sum \limits_{i=1}^{n}{\|\theta_{j}-X_{j,i}\|^{2}}}+\dfrac{W_{2}^{2}(\delta_{\theta_{j}},H_{2})}{|A^{c}|}}. \nonumber
\end{eqnarray}The above objective function suggests that the optimal solutions $\theta_{i}$, $\theta_{j}$ (equivalently, $G_{i}$ and $G_{j}$) will not be close to each other as long as $i$ and $j$ do not belong to the same set $A$ or $A^{c}$, i.e. $P_{n}^{i}$ and $P_{n}^{j}$ are very far. Therefore, the two groups of ``local'' measures $G_{j}$s do not share atoms under that setting of empirical measures.

The examples examined above indicate that the MWM problem in general does not ``encourage'' the local measures $G_{j}$s to share atoms among each other in its solution. Additionally,  when the empirical measures of local groups are very close, it may also suggest that they belong to the same cluster and the distances among optimal local measures $G_{j}$s can be very small. 

\subsubsection{Algorithm description} \label{Section:mwm_algorithm}
Now we are ready to describe our algorithm in the general case. As we mentioned in Section~\ref{Section:prelim}, direct computation of the second order Wasserstein distance is expensive. Therefore, we use the entropic regularized second order Wasserstein $\hat{W}_{2}$ to approximate the Wasserstein distance (see equation~\eqref{eq:entropic_Wasserstein} for the definition). As the regularized parameter in the entropic version of second order Wasserstein distance is fixed in our experiment, we will not specifically mention that constant in $\hat{W}_{2}$ for the simplicity of the presentation. For the algorithmic development, we consider finding a local minimum of the entropic version of problem~\eqref{eqn:multilevel_K_means_typeone_first}, which is given by:
\begin{eqnarray}
\mathop {\inf }\limits_{G_{j} \in \mathcal{O}_{k_{j}}(\Theta), \Hbold}
{\mathop {\sum }\limits_{j=1}^{m}{\hat W_{2}^{2}(G_{j},P_{n_{j}}^{j})}+\dfrac{\lambda}{m} d_{\hat W_{2}}^{2}(G_{j},\Hbold)}, \label{eqn:multilevel_K_means_entropic}
\end{eqnarray}
where $d_{\hat W_{2}}^{2}(G,\Hbold) := \mathop {\min } \limits_{1 \leq i \leq M}{\hat W_{2}^{2}(G,H_{i})}$ and $\Hbold=(H_{1},\ldots,H_{M})$, with each $H_i \in \mathcal{P}_2(\Theta)$. We refer to objective function~\eqref{eqn:multilevel_K_means_entropic} as \emph{entropic MWM}. The procedure for finding such local minimum of problem \eqref{eqn:multilevel_K_means_entropic}
is summarized in Algorithm \ref{alg:multilevels_Wasserstein_means}.
\begin{algorithm}[t!]
   \caption{Multilevel Wasserstein Means (MWM)}
   \label{alg:multilevels_Wasserstein_means}
\begin{algorithmic}
   \STATE {\bfseries Input:} data $X_{j,i}$, parameters $k_{j}$ and $M$.
   \STATE {\bfseries Output:} probability measures $G_{j}$ and elements $H_{i}$ of $\Hbold$.
   \STATE Initialize measures $G_{j}^{(0)}$, elements $H_{i}^{(0)}$ of $\Hbold^{(0)}$, $t=0$.
   \WHILE{ $G_{j}^{(t)}, H_{i}^{(t)}$ have not converged}
   \STATE 1. Update  $G_{j}^{(t)}$ for $1 \leq j \leq m$:
   \FOR{$j=1$ {\bfseries to} $m$}
   \STATE $i_{j} \leftarrow \mathop {\arg \min}\limits_{1 \leq u \leq M}{\hat W_{2}^{2}(G_{j}^{(t)},H_{u}^{(t)})}$.
   \STATE $G_{j}^{(t+1)} \leftarrow \mathop {\arg \min }\limits_{G_{j} \in \mathcal{O}_{k_{j}}(\Theta)}{\hat W_{2}^{2}(G_{j},P_{n_{j}}^{j})}+ \lambda \hat W_{2}^{2}(G_{j},H_{i_{j}}^{(t)})/m$.
   \ENDFOR
   \STATE 2. Update $H_{i}^{(t)}$ for $1 \leq i \leq M$:
   \FOR{$j=1$ {\bfseries to} $m$}
   \STATE $i_{j} \leftarrow \mathop {\arg \min}\limits_{1 \leq u \leq M}{\hat W_{2}^{2}(G_{j}^{(t+1)},H_{u}^{(t)})}$.
   \ENDFOR
   \FOR{$i=1$ {\bfseries to} $M$}
   \STATE $C_{i} \leftarrow \left\{l: i_{l}=i\right\}$ for $1 \leq i \leq M$.
   \STATE $H_{i}^{(t+1)} \leftarrow \mathop {\arg \min }\limits_{H_{i} \in \mathcal{P}_{2}(\Theta)}{\sum \limits_{l \in C_{i}}{\hat W_{2}^{2}(H_{i}, G_{l}^{(t+1)})}}$.
   \ENDFOR
   \STATE 3. $t \leftarrow t+1$.
   \ENDWHILE
\end{algorithmic}
\end{algorithm}
We explain the following details regarding the initialization and update steps required by the algorithm: 
\begin{itemize}
\item The initialization of local measures $G_{j}^{(0)}$ (i.e., the initialization of their atoms and weights) can be 
obtained by performing $K$-means clustering on local data $X_{j,i}$ for $1 \leq j \leq m$.
The initialization of elements $H_{i}^{(0)}$ of $\Hbold^{(0)}$ is based on 
a simple extension of the K-means algorithm called three-stage K-means. Details are given in Algorithm \ref{alg:three_stages_K_means} in Appendix~\ref{sec:appendix_d};
\item The update of $G_{j}^{(t+1)}$ can be computed efficiently by simply using algorithms from \citep{Cuturi-2014} 
to search for local solutions of these barycenter problems within the space $\mathcal{O}
_{k_{j}}(\Theta)$ from the atoms and weights of 
$G_{j}^{(t)}$; 
\item Since all $G_{j}^{(t+1)}$s are finite discrete 
measures, finding the update for $H_{i}^{(t+1)}$ over the whole space $\mathcal{P}_{2}
(\Theta)$ can be computationally expensive. We indeed utilize a result with Wasserstein barycenter that we can reduce this optimization problem with the Wasserstein barycenter to search for a local solution within the space $\mathcal{O}_{l^{(t)}}$,
where $l^{(t)}=\sum \limits_{j \in C_{i}}{|\text{supp}(G_{j}^{(t+1)})|}-|C_{i}|$ from the global atoms $H_{i}^{(t)}$ of $\Hbold^{(t)}$
(the justification of this reduction is derived from Theorem \ref{theorem:upperbound_barycenter} in Appendix~\ref{sec:appendix_a}). 
Motivated from this result with the Wasserstein barycenter, we also only find the update for $H_{i}^{(t+1)}$ within the space $\mathcal{O}_{l^{(t)}}$ for the entropic version of Wasserstein barycenter. This again can be done by utilizing algorithms from \citep{Cuturi-2014}. Note that, as $l^{(t)}$ becomes very large when $m$ is large, to speed up the computation of Algorithm 
\ref{alg:multilevels_Wasserstein_means} we indeed impose a threshold $L$, e.g., $L=10$, for $l^{(t)}$ in the implementation. 
\end{itemize}
The following guarantee for Algorithm~\ref{alg:multilevels_Wasserstein_means}
can be established:
\begin{thm}\label{theorem:local_convergence_multilevel_Kmeans}
For any $\lambda>0$, the values of objective function \eqref{eqn:multilevel_K_means_entropic} of the entropic MWM formulation at the updates $(G_{1}^{(t)}, G_{2}^{(t)}, \ldots, G_{m}^{(t)}, \bold{H}^{(t)})$ of Algorithm \ref{alg:multilevels_Wasserstein_means} are decreasing.
\end{thm}
\noindent
The proof of Theorem~\ref{theorem:local_convergence_multilevel_Kmeans} is in Appendix~\ref{sec:appendix_b}.
\subsection{Multilevel Wasserstein means with sharing} \label{Section:constraint_multilevels_Kmeans}
As we have observed from the analysis of several specific cases, the MWM formulation may not encourage sharing components locally among $m$ groups in its solution. However, enforced sharing has been demonstrated to be a very useful technique, which leads to the ``borrowing of strength'' among different parts of the model, consequently improving the inference efficiency~\citep{Teh-etal-06,Nguyen-2016}. In this section, we seek to encourage the borrowing of strength among groups by imposing additional constraints on the atoms of $G_{1},\ldots,G_{m}$ in the original MWM formulation \eqref{eqn:multilevel_Kmeans_typeone}. Denote $\mathcal{A}_{M,
\mathcal{S}_{K}}=\biggr\{G_{j} \in \mathcal{O}_{K}(\Theta): \text{supp}(G_{j}) \subseteq \mathcal{S}_{K}\ \forall 1 \leq j \leq 
m \biggr\}$
for any given $K, M \geq 1$, where the constraint set $\mathcal{S}_{K}$ has exactly $K$ 
elements. To simplify the exposition, let us assume that $k_{j}=K$ for all $1 
\leq j \leq m$. Consider the following locally constrained version of the MWM problem
\begin{eqnarray}
\mathop {\inf }{\mathop {\sum }\limits_{j=1}^{m}{W_{2}^{2}(G_{j},P_{n_{j}}^{j})}}
+ \lambda W_{2}^{2}(\Hcal,\dfrac{1}{m}\mathop {\sum }\limits_{j=1}^{m}{\delta_{G_{j}}}), \label{eqn:local_constraint_multilevels_Kmeans_typeone}
\end{eqnarray}
where $\mathcal{S}_{K}, G_{j} \in \mathcal{A}_{M,
\mathcal{S}_{K}},  \ \Hcal \in \mathcal{E}_{M}
(\mathcal{P}(\Theta))$ in the above infimum. We call the above optimization the problem of \emph{Multilevel Wasserstein Means with Sharing (MWMS)}. The local constraint assumption $\text{supp}(G_{j})\subseteq \mathcal{S}_{K}$ had been 
utilized previously in the literature --- see, for example, \citep{Kulis-2012}, in which the authors developed an optimization-based approach to the inference of the hierarchical Dirichlet process~\citep{Teh-etal-06},
which also encourages explicitly the sharing of local group means among local clusters. 
Now, we can rewrite objective function 
\eqref{eqn:local_constraint_multilevels_Kmeans_typeone} as follows
\begin{eqnarray}
\mathop {\inf } {\mathop {\sum }\limits_{j=1}^{m}{W_{2}^{2}(G_{j},P_{n_{j}}^{j})}+\dfrac{\lambda}{m}d_{W_{2}}^{2}(G_{j},\Hbold)}, \label{eqn:local_constraint_multilevels_Kmeans_typetwo}
\end{eqnarray} where  $\mathcal{S}_{K}, G_{j} \in \mathcal{B}_{M,
\mathcal{S}_{K}}, \Hbold=(H_{1},\ldots,H_{M})$ in the above infimum with $\mathcal{B}_{M,\mathcal{S}_{K}}=\biggr\{G_{j} \in \mathcal{O}_{K}(\Theta): 
\text{supp}(G_{j})\subseteq \mathcal{S}_{K}\ \forall 1 \leq j \leq m \biggr\}$. 

Similar to the multilevel Wasserstein means, for the algorithmic development, we also consider an entropic regularized version of MWMS, which is given by
\begin{align}
	\mathop {\inf } {\mathop {\sum }\limits_{j=1}^{m}{\hat{W}_{2}^{2}(G_{j},P_{n_{j}}^{j})}+\dfrac{\lambda}{m}d_{\hat{W}_{2}}^{2}(G_{j},\Hbold)}, \label{eqn:local_constraint_multilevels_Kmeans_typetwo_entropic}
\end{align}
where  $\mathcal{S}_{K}, G_{j} \in \mathcal{B}_{M,
\mathcal{S}_{K}}, \Hbold=(H_{1},\ldots,H_{M})$ in the infimum. We refer to objective function~\eqref{eqn:local_constraint_multilevels_Kmeans_typetwo_entropic} as \emph{entropic MWMS}.

The high level idea of finding local minimums of the objective function \eqref{eqn:local_constraint_multilevels_Kmeans_typetwo_entropic} is to first, update the elements of the constraint set $\mathcal{S}_{K}$ to provide the supports for local measures $G_{j}$'s and then, obtain the weights of these measures as well as the elements of the global set $H$ by computing the appropriate Wasserstein barycenters.  We present the pseudocode of finding the local minimum of the entropic MWMS in Algorithm~\ref{alg:local_constraint_multilevels_Wasserstein_means}. We make the following remarks regarding the initialization and updates of Algorithm \ref{alg:local_constraint_multilevels_Wasserstein_means}: 
\begin{itemize}
\item[(i)] An efficient way to initialize global set $S_{K}^{(0)}=\biggr\{a_{1}^{(0)},\ldots,a_{K}^{(0)}\biggr\} \in 
\mathbb{R}^{d \times K}$ is to perform $K$-means on the whole data set $X_{j,i}$ for $1 
\leq j \leq m, 1 \leq i \leq n_{j}$; 
\item[(ii)] For any $1 \leq j \leq K$, the updates $a_{j}^{(t+1)}$ are indeed the solutions of the following optimization problems (see the proof of Theorem~\ref{theorem:local_convergence_local_constraint_Kmeans} in Appendix~\ref{sec:appendix_b} for detailed argument about how these optimization problems are derived):
\begin{eqnarray}
\min_{a_{j}} \left\{ m \sum \limits_{u=1}^{m}{\sum \limits_{v=1}^{n_{j}}{T_{j,v}^{u}\|a_{j}-X_{u,v}\|^{2}}}
+\lambda \sum \limits_{u=1}^{m}{\sum \limits_{v}{U_{j,v}^{u}\|a_{j}-h_{i_{j},v}^{(t)}||^{2}}}\right\},\nonumber
\end{eqnarray}
where $T^{j}$ is an optimal coupling of $G_{j}^{(t)}$, $P_{n}^{j}$ and $U^{j}$ is an optimal coupling of $G_{j}^{(t)}$, $H_{i_{j}}^{(t)}$. By taking the first order derivative of the above function with respect to $a_{j}^{(t)}$, we quickly achieve $a_{j}^{(t+1)}$ as the closed form minimum of that function;
\item[(iii)] Updating the local weights of $G_{j}^{(t+1)}$ is equivalent to updating $G_{j}^{(t+1)}$ as the atoms of $G_{j}^{(t+1)}$ are known to stem from $S_{K}^{(t+1)}$. 
\end{itemize}
\begin{algorithm}[t!]
   \caption{Multilevel Wasserstein Means with Sharing (MWMS)}
   \label{alg:local_constraint_multilevels_Wasserstein_means}
\begin{algorithmic}
   \STATE {\bfseries Input:} Data $X_{j,i}$, $K$, $M$, $\lambda$.
   \STATE {\bfseries Output:} global set $S_{K}$, local measures $G_{j}$, and elements $H_{i}$ of $\Hbold$.
   \STATE Initialize $S_{K}^{(0)}=\left\{a_{1}^{(0)},\ldots,a_{K}^{(0)}\right\}$, measures $G_{j}^{(0)}$ from $S_{K}^{(0)}$, elements $H_{i}^{(0)}$ of $\Hbold^{(0)}$, and $t = 0$.
   \WHILE{$S_{K}^{(t)},G_{j}^{(t)},H_{i}^{(t)}$ have not converged}
   \STATE 1. Update global set $S_{K}^{(t)}$:
   \FOR{$j=1$ {\bfseries to} $m$}
   \STATE $i_{j} \leftarrow \mathop {\arg \min}\limits_{1 \leq u \leq M}{\hat W_{2}^{2}(G_{j}^{(t)},H_{u}^{(t)})}$.
   \STATE $T^{j} \leftarrow$ optimal coupling of $G_{j}^{(t)}$, $P_{n}^{j}$, \ $U^{j} \leftarrow$ optimal coupling of $G_{j}^{(t)}$, $H_{i_{j}}^{(t)}$.
   \ENDFOR
   \FOR{$i=1$ {\bfseries to} $M$}
   \STATE $h_{i}^{(t)} \leftarrow$ atoms of $H_{i}^{(t)}$ with $h_{i,v}^{(t)}$ as v-th column.
   \ENDFOR
   \FOR {$i=1$ {\bfseries to} $K$}
   \STATE $(m+\lambda) D \leftarrow m \sum \limits_{u=1}^{m}{\sum \limits_{v=1}^{n_{i}}{T_{i,v}^{u}}}+\lambda \sum \limits_{u=1}^{m}{\sum \limits_{v \neq i}{U_{i,v}^{u}}}$.
   \STATE $a_{i}^{(t+1)} \leftarrow \biggr(m \sum \limits_{u=1}^{m}{\sum \limits_{v=1}^{n_{i}}{T_{i,v}^{u}X_{u,v}}}+\lambda \sum \limits_{u=1}^{m}{\sum \limits_{v}{U_{i,v}^{u}h_{j_{u},v}^{(t)}}}\biggr)/mD$.
	\ENDFOR
	\STATE 2. Update $G_{j}^{(t)}$ for $1 \leq j \leq m$:
	\FOR{$j=1$ {\bfseries to} $m$}
	\STATE $G_{j}^{(t+1)} \leftarrow \mathop {\arg \min}\limits_{G_{j}: \text{supp}(G_{j}) \equiv \mathcal{S}_{K}^{(t+1)}}{\hat W_{2}^{2}(G_{j},P_{n_{j}}^{j})} +\lambda \hat W_{2}^{2}(G_{j},H_{i_{j}}^{(t)})/m$.
	\ENDFOR
   \STATE 3. Update $H_{i}^{(t)}$ for $1 \leq i \leq M$ as in Algorithm \ref{alg:multilevels_Wasserstein_means}.
   \STATE 4. $t \leftarrow t+1$.
   \ENDWHILE
\end{algorithmic}
\end{algorithm}
Now, similarly to Theorem \ref{theorem:local_convergence_multilevel_Kmeans}, we also have the following theoretical guarantee regarding the behavior of Algorithm \ref{alg:local_constraint_multilevels_Wasserstein_means}.
\begin{thm} \label{theorem:local_convergence_local_constraint_Kmeans} For any $\lambda>0$, 
the values of objective function \eqref{eqn:local_constraint_multilevels_Kmeans_typetwo_entropic} of the entropic MWMS formulation at the updates $(S_{K}^{(t)}, G_{1}^{(t)}, G_{2}^{(t)}, \ldots, G_{m}^{(t)}, \bold{H}^{(t)})$ of  Algorithm 
\ref{alg:local_constraint_multilevels_Wasserstein_means}  are decreasing.
\end{thm}
\noindent
The proof of Theorem~\ref{theorem:local_convergence_local_constraint_Kmeans} is in Appendix~\ref{sec:appendix_b}.

\section{Extension to multilevel structure data with context} \label{Section:multilevel_Wasserstein_context}
So far, we have considered the setting of multilevel structure data without any additional structure. However, complex multilevel data in practice usually include more structures. In this section, we consider a specific setting of multilevel data with context \citep{Vu-2014, Viet-2016} and develop efficient methods to cluster these data based on the idea of Wasserstein means as in the previous sections. Assume now that we have $m$ groups of $n_{j}$ exchangeable data points $X_{j,i}$ where $1 
\leq j \leq m, 1 \leq i \leq n_{j}$. For each group $j$, $\phi_{j} \in \mathbb{R}^{d_{2}}$ denotes the observed group-specific context. Our goal is to utilize the group-specific context to learn about the two-level clustering structure of the data. Similarly to the setting in Section \ref{Section:multilevel_Wasserstein}, we assume that we have at most $k_{j}$ clusters of group $j$ for $j=1,\ldots, m$ and $M$ groups of global clustering and contexts.

Based on the idea of MWM developed earlier, regarding local clustering we perform K-means clustering for group $j$, for each $j=1,\ldots,m$, to find a discrete measure $G_{j} \in \Ocal_{k_{j}}(\Theta)$ that minimizes $W_{2}^{2}(G_{j},P_{n_{j}}^{j})$. Since we would like to incorporate group context $\phi_{j}$ to study the two-level clustering structure of the data, the global clustering can be expressed as 
\begin{eqnarray}
\mathop {\inf }\limits_{\mathcal{H} \in \Ecal_{M}(\mathcal{P}_{2}(\Theta) \times \mathbb{R}^{d_{2}})}{W_{2}^{2}\biggr(
\mathcal{H},\dfrac{1}{m}\sum \limits_{j=1}^{m}{\delta_{(G_{j},\phi_{j})}}\biggr)}. \nonumber
\end{eqnarray}
By combining the losses from local and global clustering, we arrive at the following objective function
\begin{eqnarray}
\mathop {\inf }\limits_{\substack {G_{j} \in \mathcal{O}_{k_{j}}(\Theta) \\ H \in \Ecal_{M}(\mathcal{P}_{2}(\Theta) \times \mathbb{R}^{d_{2}})}}{\mathop {\sum }\limits_{j=1}^{m}{W_{2}^{2}(G_{j},P_{n_{j}}^{j})}+\lambda W_{2}^{2}(H,\dfrac{1}{m}\mathop {\sum }\limits_{j=1}^{m}{\delta_{(G_{j},\phi_{j})}})}, \label{eqn:multilevel_Kmeans_context}
\end{eqnarray}
where $\lambda>0$ is a chosen penalty number. We call the above optimization the problem of \textit{Multilevel Wasserstein Means with Context (MWMC)}. Similarly to the case of MWM, the objective function of MWMC can be rewritten as follows
\begin{eqnarray}
\mathop {\inf }\limits_{\substack {G_{j} \in \mathcal{O}_{k_{j}}(\Theta), \\ H \in (\mathcal{P}_{2}(\Theta) \times \mathbb{R}^{d_{2}})^{M}}}{\mathop {\sum }\limits_{j=1}^{m}{W_{2}^{2}(G_{j},P_{n_{j}}^{j})}+\dfrac{\lambda}{m} d_{W_{2}}^{2}\biggr((G_{j},\phi_{j}),\Hbold\biggr)}, \label{eqn:multilevel_K_means_context_first}
\end{eqnarray}
where $d_{W_{2}}^{2}\left((G,\phi),\Hbold \right)=\mathop {\min } \limits_{1 \leq i \leq m}{\left\{W_{2}^{2}(G,H_{i})+\|\phi - \theta_{i}\|^{2}\right\}}$ and $\Hbold=\left\{(H_{1},\theta_{1}),\ldots,(H_{M},\theta_{M})\right\}$\\ $\in (\mathcal{P}_{2}(\Theta) \times \mathbb{R}^{d_{2}})^{M}$. 

Similar to both the MWM and MWMS, we also consider the entropic version of MWMC for the computational purpose, which we refer to as \emph{entropic MWMC}. The objective function of the entropic MWMC is given by:
\begin{align}
	\mathop {\inf }\limits_{\substack {G_{j} \in \mathcal{O}_{k_{j}}(\Theta), \\ H \in (\mathcal{P}_{2}(\Theta) \times \mathbb{R}^{d_{2}})^{M}}}{\mathop {\sum }\limits_{j=1}^{m}{\hat{W}_{2}^{2}(G_{j},P_{n_{j}}^{j})}+\dfrac{\lambda}{m} d_{\hat{W}_{2}}^{2}\biggr((G_{j},\phi_{j}),\Hbold\biggr)}. \label{eqn:multilevel_K_means_context_first_entropic}
\end{align}
The procedure for finding a local minimum of the entropic MWMC objective function \eqref{eqn:multilevel_K_means_context_first_entropic} is summarized in Algorithm \ref{alg:multilevels_Wasserstein_means_context}. Here, we have the following comments regarding the initialization and update steps of that algorithm:
\begin{itemize}
\item The initialization of local measures $G_{j}^{(0)}$ and global set $\Hbold^{(0)}$ can be carried out in a similar fashion as those in Algorithm \ref{alg:multilevels_Wasserstein_means}. Furthermore, the initialization of $\theta_{i}^{(0)}$ can be obtained by performing K-means++ clustering proposed in \citep{Arthur-2007} on the context data $\phi_{1},\ldots,\phi_{m}$.
\item The approaches to update $G_{j}^{(t+1)}$ and $\Hbold^{(t+1)}$ are similar to those in Algorithm \ref{alg:multilevels_Wasserstein_means}. The update of $\theta_{i}^{(t+1)}$ is to find an optimal center to minimize its distance to context $\phi_{l}$ for all $l \in C_{i}$.
\end{itemize}
Similar to Theorems \ref{theorem:local_convergence_multilevel_Kmeans} and~\ref{theorem:local_convergence_local_constraint_Kmeans}, we also have the following guarantee regarding the performance of Algorithm \ref{alg:multilevels_Wasserstein_means_context}:
\begin{thm} \label{theorem:local_convergence_multilevel_Wasserstein_context}
For any $\lambda>0$, the values of objective function \eqref{eqn:multilevel_K_means_context_first_entropic} of the entropic MWMC formulation at the updates $(G_{1}^{(t)}, G_{2}^{(t)}, \ldots, G_{m}^{(t)}, \bold{H}^{(t)}, \theta_{1}^{(t)},\ldots, \theta_{M}^{(t)})$ of  Algorithm 
\ref{alg:multilevels_Wasserstein_means_context}  are decreasing.
\end{thm}
The proof of Theorem~\ref{theorem:local_convergence_multilevel_Wasserstein_context} is similar to that of Theorem~\ref{theorem:local_convergence_multilevel_Kmeans}; therefore, it is omitted. Finally, we would like to remark that the extension of the MWMC problem to the setting in which local measures $G_{j}$ share atoms among others can be carried out in a similar fashion as that in Section \ref{Section:constraint_multilevels_Kmeans}.
\begin{algorithm}[t!]
   \caption{Multilevel Wasserstein Means with Context (MWMC)}
   \label{alg:multilevels_Wasserstein_means_context}
\begin{algorithmic}
   \STATE {\bfseries Input:} data $X_{j,i}$, context $\phi_{j}$, parameters $k_{j}$ and $M$.
   \STATE {\bfseries Output:} probability measures $G_{j}$ and elements $(H_{i},\theta_{i})$ of $\Hbold$.
   \STATE Initialize measures $G_{j}^{(0)}$, elements $(H_{i}^{(0)},\theta_{i}^{(0)})$ of $\Hbold^{(0)}$, $t=0$.
   \WHILE{$G_{j}^{(t)}, H_{i}^{(t)}, \theta_{i}^{(t)}$ have not converged}
   \STATE 1. Update $G_{j}^{(t)}$ for $1 \leq j \leq m$:
   \FOR{$j=1$ {\bfseries to} $m$}
   \STATE $i_{j} \leftarrow \mathop {\arg \min}\limits_{1 \leq u \leq M}{\hat W_{2}^{2}(G_{j}^{(t)},H_{u}^{(t)})+\|\phi_{j}-\theta_{u}^{(t)}\|^{2}}$.
   \STATE $G_{j}^{(t+1)} \leftarrow \mathop {\arg \min }\limits_{G_{j} \in \mathcal{O}_{k_{j}}(\Theta)}{\hat W_{2}^{2}(G_{j},P_{n_{j}}^{j})}+\lambda \hat W_{2}^{2}(G_{j},H_{i_{j}}^{(t)})/m$.
   \ENDFOR
   \STATE 2. Update $(H_{i}^{(t)}, \theta_{i}^{(t)})$ for $1 \leq i \leq M$:
   \FOR{$j=1$ {\bfseries to} $m$}
   \STATE $i_{j} \leftarrow \mathop {\arg \min}\limits_{1 \leq u \leq M}{\hat W_{2}^{2}(G_{j}^{(t+1)},H_{u}^{(t)})+\|\phi_{j}-\theta_{u}^{(t)}\|^{2}}$.
   \ENDFOR
   \FOR{$i=1$ {\bfseries to} $M$}
   \STATE $C_{i} \leftarrow \left\{l: i_{l}=i\right\}$ for $1 \leq i \leq M$.
   \STATE $H_{i}^{(t+1)} \leftarrow \mathop {\arg \min }\limits_{H_{i} \in \mathcal{P}_{2}(\Theta)}{\sum \limits_{l \in C_{i}}{\hat W_{2}^{2}(H_{i}, G_{l}^{(t+1)})}}$.
   \STATE $\theta_{i}^{(l+1)}=(\sum \limits_{l \in C_{i}}{\phi_{l}})/|C_{i}|$.
   \ENDFOR
   \STATE 3. $t \leftarrow t+1$.
   \ENDWHILE
\end{algorithmic}
\end{algorithm}

\section{Robust multilevel clustering with first order Wasserstein distance}\label{Section:robust_mutilevel_Wasserstein_median}
In the previous sections, we develop our models based on $W_{2}$ metric. Nevertheless, these formulations do not directly account for the outliers in the data. In order to improve the robustness, it may be desirable to make use of the first order Wasserstein distance instead of the second order one. In particular, we reformulate MWM objective function in Section \ref{Section:multilevel_Wasserstein} based on $W_{1}$ distance as follows
\begin{eqnarray}
\mathop {\inf }\limits_{\substack {G_{j} \in \mathcal{O}_{k_{j}}(\Theta), \\\mathcal{H} \in 
\mathcal{E}_{M}(\mathcal{P}_{2}(\Theta))}}{\mathop {\sum }\limits_{j=1}^{m}{W_{1}(G_{j},P_{n_{j}}^{j})}}
+ \lambda W_{1}(\mathcal{H},\dfrac{1}{m}\mathop {\sum }\limits_{j=1}^{m}{\delta_{G_{j}}}), \label{eqn:robust_multilevel_Wasserstein_median_typeone}
\end{eqnarray}
where $\lambda$ is a positive number used to balance the accumulative losses between the local and global clustering. We call the above optimization the problem of 
\textit{Multilevel Wasserstein Geometric Median} (MWGM). Similarly to the equivalence between objective functions \eqref{eqn:original_Kmeans} and \eqref{eqn:multilevel_K_means_typeone_first}, we can demonstrate that the objective function  \eqref{eqn:robust_multilevel_Wasserstein_median_typeone} is equivalent to the following simpler optimization problem 
\begin{eqnarray}
\mathop {\inf }\limits_{G_{j} \in \mathcal{O}_{k_{j}}(\Theta), \Hbold}
{\mathop {\sum }\limits_{j=1}^{m}{W_{1}(G_{j},P_{n_{j}}^{j})}+\dfrac{\lambda}{m} d_{W_{1}}(G_{j},\Hbold)}, \label{eqn:robust_multilevel_Wasserstein_median_typeone_first}
\end{eqnarray}
where $d_{W_{1}}(G,\Hbold) := \mathop {\min } \limits_{1 \leq i \leq M}{W_{1}(G,H_{i})}$ and $\Hbold=(H_{1},\ldots,H_{M})$,
with each $H_i \in \mathcal{P}_2(\Theta)$. 

For the algorithmic development, we also consider an entropic version of the MWGM, which admits the following form:
\begin{align}
	\mathop {\inf }\limits_{G_{j} \in \mathcal{O}_{k_{j}}(\Theta), \Hbold}
{\mathop {\sum }\limits_{j=1}^{m}{\hat{W}_{1}(G_{j},P_{n_{j}}^{j})}+\dfrac{\lambda}{m} d_{\hat{W}_{1}}(G_{j},\Hbold)}. \label{eqn:robust_multilevel_Wasserstein_median_typeone_first_entropic}
\end{align}
We refer to the objective function~\eqref{eqn:robust_multilevel_Wasserstein_median_typeone_first_entropic} as \emph{entropic MWGM}.  Unlike our previous algorithms, the algorithm to study objective function \eqref{eqn:robust_multilevel_Wasserstein_median_typeone_first_entropic} relies on the update with Wasserstein barycenter under the entropic version of $W_{1}$ distance, which means we need to solve the following optimization problem
\begin{eqnarray}
\overline{Q}_{N,\lambda}=\mathop {\arg \min}\limits_{Q \in \mathcal{P}_{1}(\Theta)}{\sum \limits_{i=1}^{N}{\eta_{i} \hat{W}_{1}(Q,Q_{i})}}, \label{eqn:Wasserstein_barycenter_first_order}
\end{eqnarray} 
where $Q_{1},\ldots, Q_{N} \in \mathcal{P}_{1}(\Theta)$ for $N \geq 1$ and $\eta \in \Delta_{N}$ denotes weights associated with $Q_{1},\ldots, Q_{N}$.  In Appendix~\ref{sec:appendix_e}, we provide an efficient algorithm to determine the Wasserstein barycenter over $\Ocal_{k}(\Theta)$ under entropic version of $W_{1}$ distance when $Q_{i}$'s are all discrete measures with finite number of atoms. The high level idea of our algorithm is to utilize the dual transportation formulation from \citep{Cuturi-2014} to update weights of the barycenter while we use the idea of weighted geometric median to update the atoms of the barycenter. The algorithm for finding a local minimum of problem \eqref{eqn:robust_multilevel_Wasserstein_median_typeone_first_entropic} is summarized in Algorithm \ref{alg:multilevels_Wasserstein_median}.
\begin{algorithm}[t!]
   \caption{Multilevel Wasserstein Geometric Median (MWGM)}
   \label{alg:multilevels_Wasserstein_median}
\begin{algorithmic}
   \STATE {\bfseries Input:} data $X_{j,i}$, parameters $k_{j}$ and $M$.
   \STATE {\bfseries Output:} probability measures $G_{j}$ and elements $H_{i}$ of $\Hbold$.
   \STATE Initialize measures $G_{j}^{(0)}$, elements $H_{i}^{(0)}$ of $\Hbold^{(0)}$, $t=0$.
   \WHILE{$G_{j}^{(t)}, H_{i}^{(t)}$ have not converged}
   \STATE 1. Update $G_{j}^{(t)}$ for $1 \leq j \leq m$:
   \FOR{$j=1$ {\bfseries to} $m$}
   \STATE $i_{j} \leftarrow \mathop {\arg \min}\limits_{1 \leq u \leq M}{\hat W_{1}(G_{j}^{(t)},H_{u}^{(t)})}$.
   \STATE $G_{j}^{(t+1)} \leftarrow \mathop {\arg \min }\limits_{G_{j} \in \mathcal{O}_{k_{j}}(\Theta)}{\hat W_{1}(G_{j},P_{n_{j}}^{j})}+\lambda \hat W_{1}(G_{j},H_{i_{j}}^{(t)})/m$.
   \ENDFOR
   \STATE 2. Update $H_{i}^{(t)}$ for $1 \leq i \leq M$:
   \FOR{$j=1$ {\bfseries to} $m$}
   \STATE $i_{j} \leftarrow \mathop {\arg \min}\limits_{1 \leq u \leq M}{\hat W_{1}(G_{j}^{(t+1)},H_{u}^{(t)})}$.
   \ENDFOR
   \FOR{$i=1$ {\bfseries to} $M$}
   \STATE $C_{i} \leftarrow \left\{l: i_{l}=i\right\}$ for $1 \leq i \leq M$.
   \STATE $H_{i}^{(t+1)} \leftarrow \mathop {\arg \min }\limits_{H_{i} \in \mathcal{P}_{2}(\Theta)}{\sum \limits_{l \in C_{i}}{\hat W_{1}(H_{i}, G_{l}^{(t+1)})}}$.
   \ENDFOR
   \STATE 3. $t \leftarrow t+1$.
   \ENDWHILE
\end{algorithmic}
\end{algorithm}

Similarly to the previous algorithms, we also have the following guarantee regarding the performance of Algorithm \ref{alg:multilevels_Wasserstein_median}.
\begin{thm} \label{theorem:local_convergence_multilevel_median}
For any $\lambda>0$, the values of objective function \eqref{eqn:robust_multilevel_Wasserstein_median_typeone_first_entropic} of the entropic MWGM formulation at the updates $(G_{1}^{(t)}, G_{2}^{(t)}, \ldots, G_{m}^{(t)}, \bold{H}^{(t)})$ of  Algorithm 
\ref{alg:multilevels_Wasserstein_median}  are decreasing.
\end{thm}
\noindent
The proof of Theorem~\ref{theorem:local_convergence_multilevel_median} is similar to that of Theorem~\ref{theorem:local_convergence_multilevel_Kmeans}; therefore, it is omitted.
\section{Consistency results} \label{Section:consistency_multilevel_Kmeans}
We proceed to establish consistency for the estimators introduced in the previous sections. For the brevity of the presentation, we only focus on the MWM method while the consistency for MWMS, MWMC, and MWGM can be obtained in the similar fashion. The consistency of the estimators on the MWGM method is presented in Appendix~\ref{sec:appendix_f}. Fix $m$ and assume that $P^{j}$ is the true distribution of data $X_{j,i}$ for $j = 1,\ldots,m$. Write $\vec{G}=(G_{1},\ldots,G_{m})$ and $\vec{n}=(n_{1},\ldots,n_{m})$. We say $\vec{n} \to \infty$ if $n_{j} \to \infty$ for $j=1,\ldots, m$. Define the following functions
\begin{eqnarray}
f_{\vec{n}}(\vec{G},\Hcal)=\mathop {\sum }\limits_{j=1}^{m}{W_{2}^{2}(G_{j},P_{n_{j}}^{j})}+\lambda W_{2}^{2}(\Hcal,\dfrac{1}{m}\mathop {\sum }\limits_{j=1}^{m}{\delta_{G_{j}}}), \nonumber \\
f(\vec{G},\Hcal)=\mathop {\sum }\limits_{j=1}^{m}{W_{2}^{2}(G_{j},P^{j})}+\lambda W_{2}^{2}(\Hcal,\dfrac{1}{m}\mathop {\sum }\limits_{j=1}^{m}{\delta_{G_{j}}}), \nonumber
\end{eqnarray}
where $G_{j} \in \mathcal{O}_{k_{j}}(\Theta)$, $\Hcal \in \mathcal{E}_{M}(\mathcal{P}_{2}(\Theta))$ for $1 \leq j \leq m$. 
The first consistency property of the WMW formulation is as follows.
\begin{thm} \label{theorem:objective_consistency_multilevel_Wasserstein_means} 
Given that $P^{j} \in \mathcal{P}_{2}(\Theta)$ for $1 \leq j \leq m$, then the following holds:

\noindent
(i) There holds almost surely, as $\vec{n} \to \infty$
\begin{eqnarray}
\mathop {\inf }\limits_{\substack {G_{j} \in \mathcal{O}_{k_{j}}(\Theta), \\ \Hcal \in \mathcal{E}_{M}(\mathcal{P}_{2}(\Theta))}}f_{\vec{n}}(\vec{G},\Hcal) - \mathop {\inf }\limits_{\substack {G_{j} \in \mathcal{O}_{k_{j}}(\Theta), \\ \Hcal \in \mathcal{E}_{M}(\mathcal{P}_{2}(\Theta))}}f(\vec{G},\Hcal) \rightarrow 0. \nonumber
\end{eqnarray}
\noindent
(ii) If $\Theta$ is bounded, then 
\begin{align*}
\left| \mathop {\inf }\limits_{\substack {G_{j} \in \mathcal{O}_{k_{j}}(\Theta), \\ \Hcal \in \mathcal{E}_{M}(\mathcal{P}_{2}(\Theta))}}f_{\vec{n}}^{1/2}(\vec{G},\Hcal) - \mathop {\inf }\limits_{\substack {G_{j} \in \mathcal{O}_{k_{j}}(\Theta), \\ \Hcal \in \mathcal{E}_{M}(\mathcal{P}_{2}(\Theta))}}f^{1/2}(\vec{G},\Hcal) \right| = \begin{cases} O_{P}(m \cdot n_{\vee}^{-1/d}) \ \text{when} \ d \geq 5, \\
  O_{P}(m \cdot n_{\vee}^{-1/4}) \ \text{when} \ d \leq 4,
  \end{cases}
\end{align*}
where $n_{\vee} = \max_{1 \leq i \leq m} n_{i}$.
\end{thm}
The proof of Theorem~\ref{theorem:objective_consistency_multilevel_Wasserstein_means} is in Appendix~\ref{sec:appendix_b}. The next theorem establishes that the ``true'' global and local clusters can be
recovered. To this end, assume that for each $\vec{n}$ there is an optimal solution $
(\widehat{G}_{1}^{n_{1}},\ldots,\widehat{G}_{m}^{n_{m}},\widehat{\Hcal}^{\vec{n}})$ or in 
short $(\vec{\widehat{G}}^{\vec{n}},\Hcal^{\vec{n}})$ of the objective function 
\eqref{eqn:multilevel_Kmeans_typeone}. Moreover, there exist a (not necessarily unique)
optimal solution minimizing $f(\vec{G},\Hcal)$ over $G_{j} \in \mathcal{O}_{k_{j}}(\Theta)$ and $\Hcal \in 
\mathcal{E}_{M}(\mathcal{P}_{2}(\Theta))$. Let $\mathcal{F}$ be the collection 
of such optimal solutions. For any $G_{j} \in \mathcal{O}_{k_{j}}(\Theta)$ and $\Hcal \in 
\mathcal{E}_{M}(\mathcal{P}_{2}(\Theta))$, we define
\begin{eqnarray}
d(\vec{G},\Hcal,\mathcal{F})=\inf \limits_{(\vec{G}^{0}, \Hcal^{0}) \in \mathcal{F}}\sum \limits_{j=1}^{m}{W_{2}^{2}(G_{j},G_{j}^{0})} \nonumber
+W_{2}^{2}(\Hcal,\Hcal^{0}). \nonumber
\end{eqnarray}
Given the above assumptions, we have the following result regarding the convergence of $(\widehat{\vec{G}}^{\vec{n}},\Hcal^{\vec{n}})$:
\begin{thm} \label{theorem:convergence_measures_multilevel_Wasserstein_means}
Assume that $\Theta$ is bounded and $P^{j} \in \mathcal{P}_{2}(\Theta)$ for all $1 \leq j 
\leq m$. Then, we have $d(\vec{\widehat{G}}^{\vec{n}},\widehat{\Hcal}
^{\vec{n}},\mathcal{F}) \to 0$ as $\vec{n} \to \infty$ almost surely.
\end{thm}
\noindent
The proof of Theorem~\ref{theorem:convergence_measures_multilevel_Wasserstein_means} is in Appendix~\ref{sec:appendix_b}. 
\paragraph{Remark:}(i) The assumption that $\Theta$ is bounded is just for the convenience of 
proof argument. From the work of~\cite{Pollard-1981}, the consistency of K-means clustering is established under no assumptions on the compactness of the parameters. The idea of this paper is to show that when the sample size is sufficiently large, the optimal clusters are contained in a ball centered at the origin of sufficiently large radius. From here, we can reduce the consistency analysis to compact subset argument with the cluster centers. To extend the result of Theorem~\ref{theorem:convergence_measures_multilevel_Wasserstein_means} when $\Theta = \mathbb{R}^{d}$, we also need a similar step to prove that the supports of local and global measures from multilevel Wasserstein means will be in a ball of sufficiently large radius when the sample sizes of local groups and the number of groups are sufficiently large. It is beyond the proof technique of the current paper and we leave it for the future work. (ii) If $|\mathcal{F}|=1$, i.e. there exists a unique optimal solution $(\vec{G}^{0},\Hcal^{0})$ minimizing $f(\vec{G},\Hcal)$ over $G_{j} \in \mathcal{O}_{k_{j}}(\Theta)
$ and $\Hcal \in \mathcal{E}_{M}(\mathcal{P}_{2}(\Theta))$, the result of Theorem 
\ref{theorem:convergence_measures_multilevel_Wasserstein_means} implies that $W_{2}
(\widehat{G}_{j}^{n_{j}},G_{j}^{0}) \to 0$ for $1 \leq j \leq m$ and $W_{2}(\widehat{\Hcal}
^{\vec{n}},\Hcal^{0}) \to 0$ as $\vec{n} \to \infty$.

\section{Empirical studies} \label{Section:data_analysis}
In this section, we present extensive simulation studies with our models under both synthetic data (Section~\ref{sec:experiment_synthetic}) and real data (Section~\ref{sec:experiment_real})\footnote{Code is available at \url{https://github.com/viethhuynh/wasserstein-means}}. In our experiments, we used a fixed value of all entropic regularization parameters $\tau=10$. For the regularized term $\lambda$, we heuristically choose to balance global and local terms, i.e., $\lambda \approx  \nicefrac{W_{2}^{2}(\mathcal{H},\nicefrac{1}{m}\mathop {\sum _{j=1}^{m}{\delta_{G_{j}}})}}{\sum_{j=1}^{m}{W_{2}^{2}(G_{j},P_{n_{j}}^{j})}}$. Before presenting experimental results, we describe our methodology to scale up our algorithms using a distributed computation framework in the following section.
\subsection{Parallel implementation with Apache Spark} \label{sec:experiment_spark}
The running time complexity of the MWM and MWMS algorithms will be increased linearly with $m$ --- the number of data groups. When the number of data groups is large (e.g., tens of thousands or millions), the running time for the learning routine in these algorithms is dramatically increased. One possible solution to adapt MWM and MWMS algorithms for large-scale settings is to parallelize the learning process. Fortunately, Apache Spark provides an elegant framework to help us to accelerate running time with the map-reduce mechanism. We use the Apache Spark framework to simultaneously update the clustering index and atoms of each data group in these algorithms. We summarize the pseudo-code of Map-Reduce procedures for MWM in Algorithm~\ref{alg:mapreduce}. Our parallelized algorithm can speed up the learning algorithms up to some orders-of-magnitude, which allows us to scale up the learning problem toward large real-world datasets containing millions of groups. The orders-of-magnitude speedup depends on the number of processors (or cores) that the systems process. For instance, if we have a cluster with hundreds or thousands of cores, the orders of magnitude can be to 2 or 3. Our implementation in this paper focuses on a large number of groups, therefore, it does not help to improve running time when the number of supports is large. However, to deal with a large number of supports, we can use the GPU implementation of Cuturi's algorithms. Since Apache Spark supports distributed systems running with GPU platforms, we can obtain scalable algorithms with a large number of groups as well as the vast size of the supports.

\begin{algorithm}[t!]
   \caption{MapReduce for Multilevel Wasserstein Means (MWM)}
   \label{alg:mapreduce}
\begin{algorithmic}
    \STATE {\bfseries method} MAP (j,$G_j^{(t)}, \Hbold^{(t)}$)
     \STATE \qquad{} \qquad{} $i_{j} \leftarrow \mathop {\arg \min}\limits_{1 \leq u \leq M}{\hat W_{2}^{2}(G_{j}^{(t)},H_{u}^{(t)})}$.
   \STATE \qquad{} \qquad{} $G_{j}^{(t+1)} \leftarrow \mathop {\arg \min }\limits_{G_{j} \in \mathcal{O}_{k_{j}}(\Theta)}{\hat W_{2}^{2}(G_{j},P_{n_{j}}^{j})}+ \lambda \hat W_{2}^{2}(G_{j},H_{i_{j}}^{(t)})/m$.
    \STATE \qquad{} \qquad{} return ($i_{j},G_{j}^{(t+1)}$ )
   \STATE
    \STATE {\bfseries method} REDUCE ($\{i_j\},\{G_j^{(t+1)}\}$)
       \STATE \qquad{} \qquad{}{\bfseries for} {$i=1$ {\bfseries to} $M$}
   \STATE\qquad{} \qquad{} \qquad{}$C_{i} \leftarrow \left\{l: i_{l}=i\right\}$ for $1 \leq i \leq M$.
   \STATE\qquad{} \qquad{} \qquad{}$H_{i}^{(t+1)} \leftarrow \mathop {\arg \min }\limits_{H_{i} \in \mathcal{P}_{2}(\Theta)}{\sum \limits_{l \in C_{i}}{\hat W_{2}^{2}(H_{i}, G_{l}^{(t+1)})}}$.
\STATE \qquad{} \qquad{}{\bfseries endfor}    
\STATE \qquad{} \qquad{}return \Hbold
  
\end{algorithmic}
\end{algorithm}

\subsection{Synthetic data}
\label{sec:experiment_synthetic}
\begin{figure*}[ht]
\centerline{\includegraphics[width=0.99\textwidth]{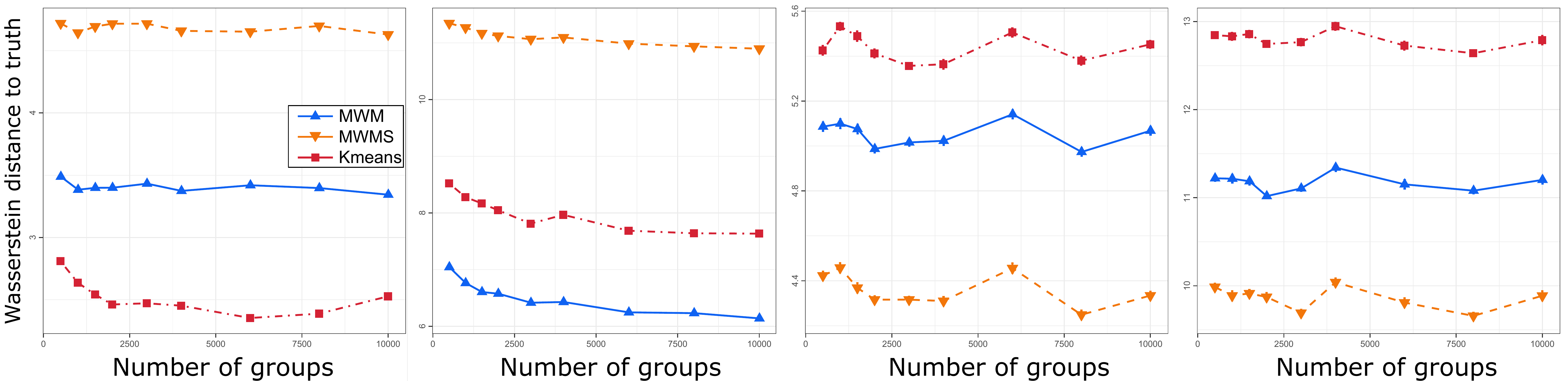}}
\caption{Data with a lot of small groups: (a) NC data with constant variance; (b) NC data with non-constant variance; (c)
LC data with constant variance; (d) LC data with non-constant variance}
\label{fig:simul_M}
\end{figure*}

\begin{figure*}[ht]
\centerline{\includegraphics[width=0.99 \textwidth]{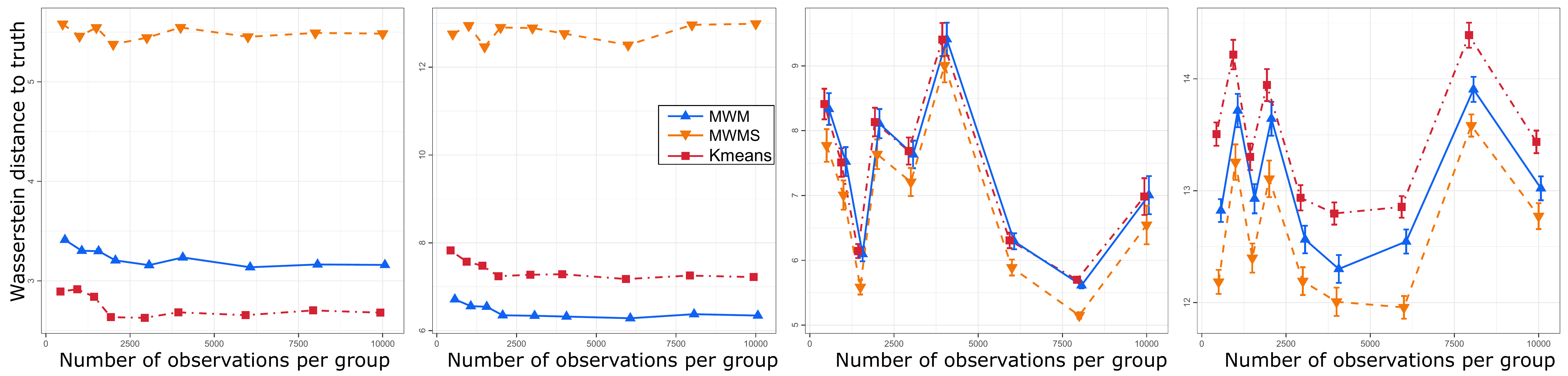}}
\caption{Data with a few big groups: (a) NC data with constant variance; (b) NC data with non-constant variance; (c) LC
data with constant variance; (d) LC data with non-constant variance}
\label{fig:simul_N}
\end{figure*}
First, we are interested in evaluating the effectiveness of all clustering algorithms in the paper by considering different synthetic data generating processes. Unless otherwise specified, we set the number of groups $m=50$, number of observations per group $n_{j}=50$, dimensions of each observation $d=10$, number of global clusters $M=5$ with 6 atoms. 
\paragraph{Multilevel Wasserstein means:} For Algorithm \ref{alg:multilevels_Wasserstein_means} (MWM), local measures $G_{j}$ have 5 atoms each. We call this no-constraint (NC) data; for Algorithm  \ref{alg:local_constraint_multilevels_Wasserstein_means} (MWMS) the number of atoms in the constraint set $S_K$ is 50. These atoms are shared among local measures $G_{j}$ each of which contains 5 atoms. We call this local-constraint (LC) data. As a benchmark for the comparison, we will use a basic 3-stage K-means approach, (the details of data generation and 3-stage K-means algorithm can be found in Appendix~\ref{sec:appendix_d}). The Wasserstein distances between the estimated distributions (i.e. $\hat 
G_1,\ldots,\hat G_m$; $\hat H_1,\ldots,\hat H_M$) and the data generating ones will be used as the comparison metric. 

Recall that the MWM formulation does not impose constraints on the atoms of $G_{i}$ whilst the MWMS formulation explicitly enforces the sharing of atoms across these measures. We use multiple layers of mixtures while adding Gaussian noise at each layer to generate global and local clusters and the no-constraint (NC) data. We vary the number of groups $m$ from 500 to 10,000. We notice that the 3-stage K-means algorithm performs best when there is no constraint structure \emph{and} the variance is constant across all clusters (Figures 
\ref{fig:simul_M}a and \ref{fig:simul_N}a) --- this is, not surprisingly, a favorable setting for the basic K-means method. As soon as we depart from the (unrealistic) constant variance no-sharing assumption, our MWM algorithm starts to outperform the basic 3-stage K-means (Figures 1b and 2b).
The superior performance is most prominent with local-constraint (LC) data (with or without constant variance condition) 
(see Figures \ref{fig:simul_M}(c, d)).
It is worth noting that even when the group variances are constant, the 3-stage K-means is no longer
effective because it fails to account for the shared structure. 
When $m=50$ and group sizes are larger, we set $S_K=15$. The results reported in 
Figures \ref{fig:simul_N}(c, d) demonstrate the effectiveness and flexibility of our algorithms.
\begin{figure}[ht]
\centering{}\includegraphics[width=0.99\textwidth]{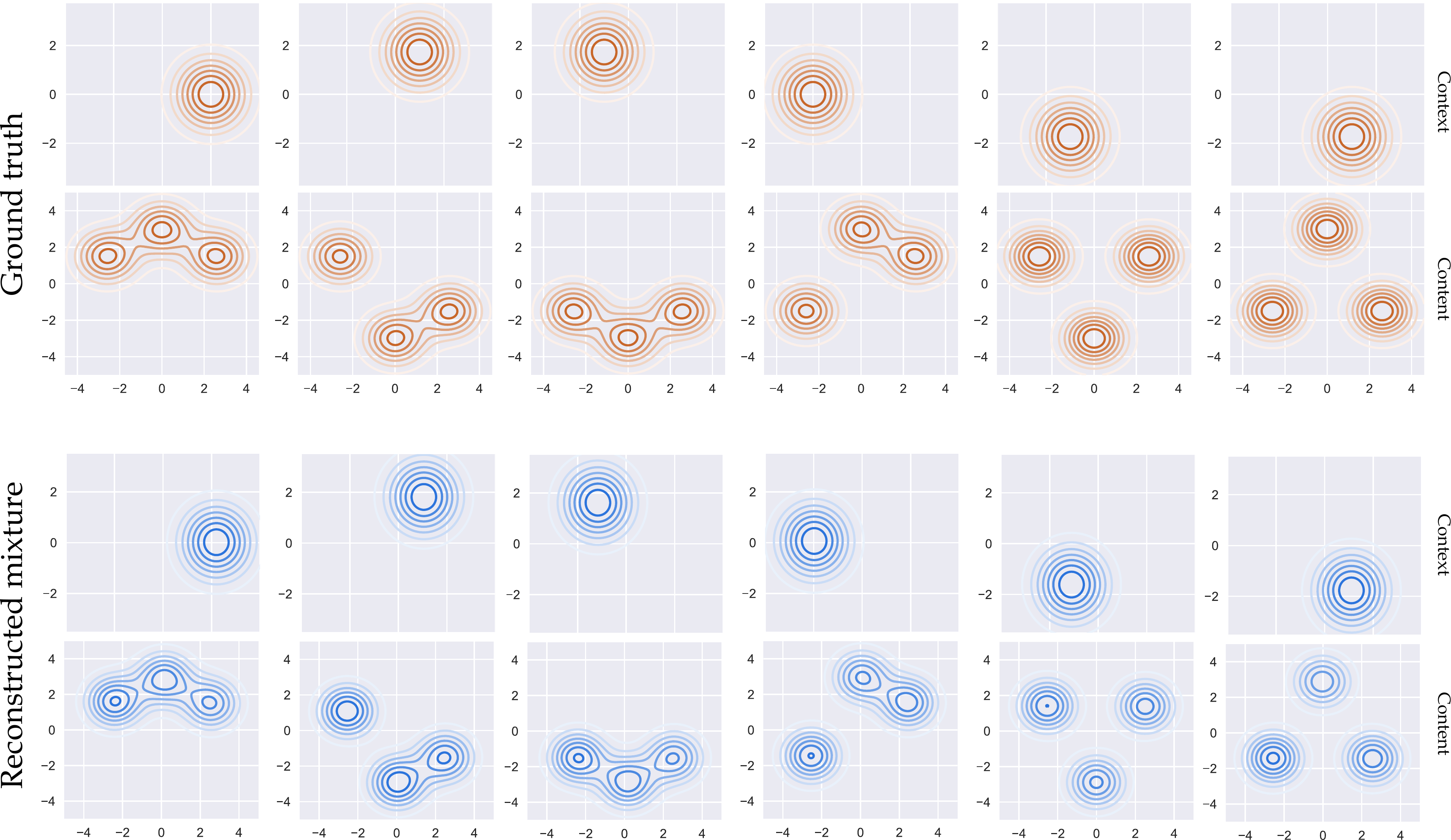}\caption{Synthetic data with context.\label{fig: synthetic_context}}

\end{figure}
\paragraph{Multilevel Wasserstein means with context:} 
Now, we demonstrate the capability of the MWMC framework to model the synthetic multilevel data with context. There are six clusters, each of which is denoted as a column in the ground truth row in Figure \ref{fig: synthetic_context}. In each cluster, the content data (bottom square) is selected as a mixture of three Gaussian components from a set of six shared Gaussian components whilst the context data (top square) is a Gaussian distribution selected from six predefined Gaussian components. Visually, the top two rows in Figure \ref{fig: synthetic_context} show the ground truth data including context (a Gaussian distribution) and content (a Gaussian mixture model). We uniformly generate 3000 groups of data from the above six clusters. Each group belongs to one of the six aforementioned clusters. Once the clustering index of a data group has been determined, we generate 100 data points from the corresponding mixture of Gaussians and a corresponding context observation. We run the synthetic data with the MWMC algorithm. The bottom two rows in Figure \ref{fig: synthetic_context} depict the reconstructed context and content data which are similar to the ground truth.


\subsection{Real-world data analysis}
\label{sec:experiment_real}
We now apply our multilevel clustering algorithms to two real-world datasets: LabelMe and StudentLife. 

\textbf{LabelMe dataset}\footnote{http://labelme.csail.mit.edu} consists of $2,688$ annotated images which are classified into 8 scene categories including \emph{tall buildings, inside city, street, highway, coast, open country, mountain,} and \emph{forest} \citep{Oliva-2001}. Each image contains multiple annotated regions. Each region, which is annotated by users, represents an object in the image.  As shown in Figure \ref{fig: LabelMeExample}, the left image is an example from \emph{open country} category and contains 4 regions while the right panel shows an image of \emph{tall buildings} category including 16 regions. Note that the regions in each image can
be overlapped. Removing the images containing less than 4 regions, we obtain $1,800$ images for our experiments. We then extract the GIST feature \citep{Oliva-2001} for each region in an image. GIST is a 512-dimensional visual descriptor representing perceptual dimensions and oriented spatial structures of a scene. We further use PCA to project GIST features onto 30 dimensions. Consequently, we obtain $1,800$ ``documents'', each of which contains regions as observations. Each region is represented by a 30-dimensional vector. We now can perform clustering regions in every image since they are visually correlated. In the next level of clustering, we can cluster images into scene categories. Each image in the LabelMe dataset has associated tags. These tags of each image are represented as a count context vector
\begin{figure}[ht]
\centering{}\includegraphics[width=0.17\textwidth]{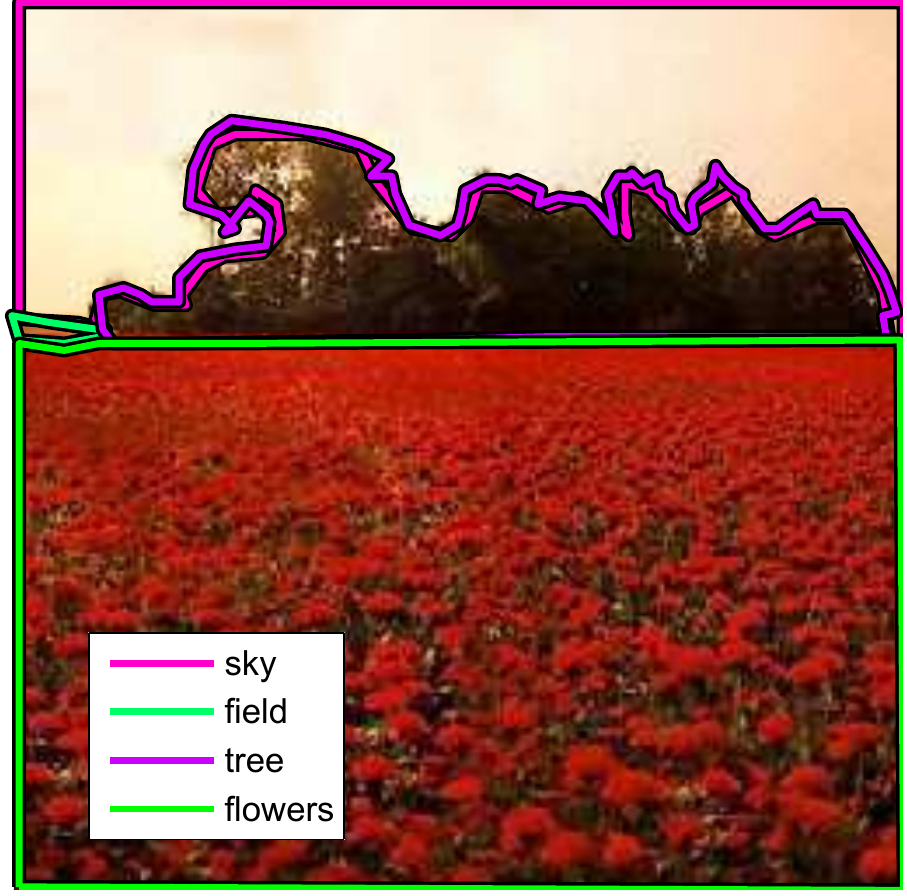}\qquad{}\qquad{}\includegraphics[width=0.22\textwidth]{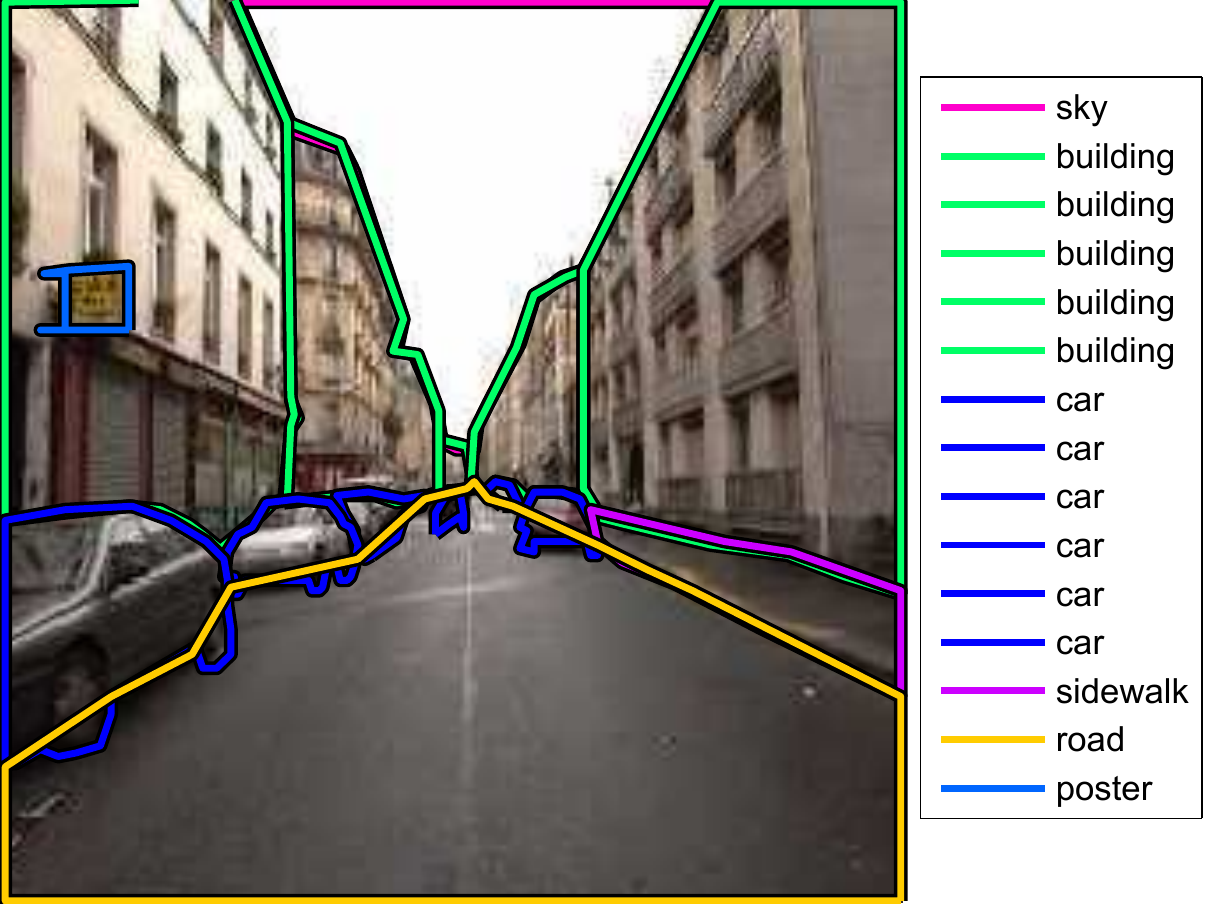}\caption{Sample images in LabelMe dataset. Each image consists of different number of  annotated regions.\label{fig: LabelMeExample}}
\end{figure}
\begin{figure*}[t]
\begin{centering}
\subfloat[\label{fig:labelme_clusters}]{\begin{centering}
\includegraphics[width=0.2\paperwidth]{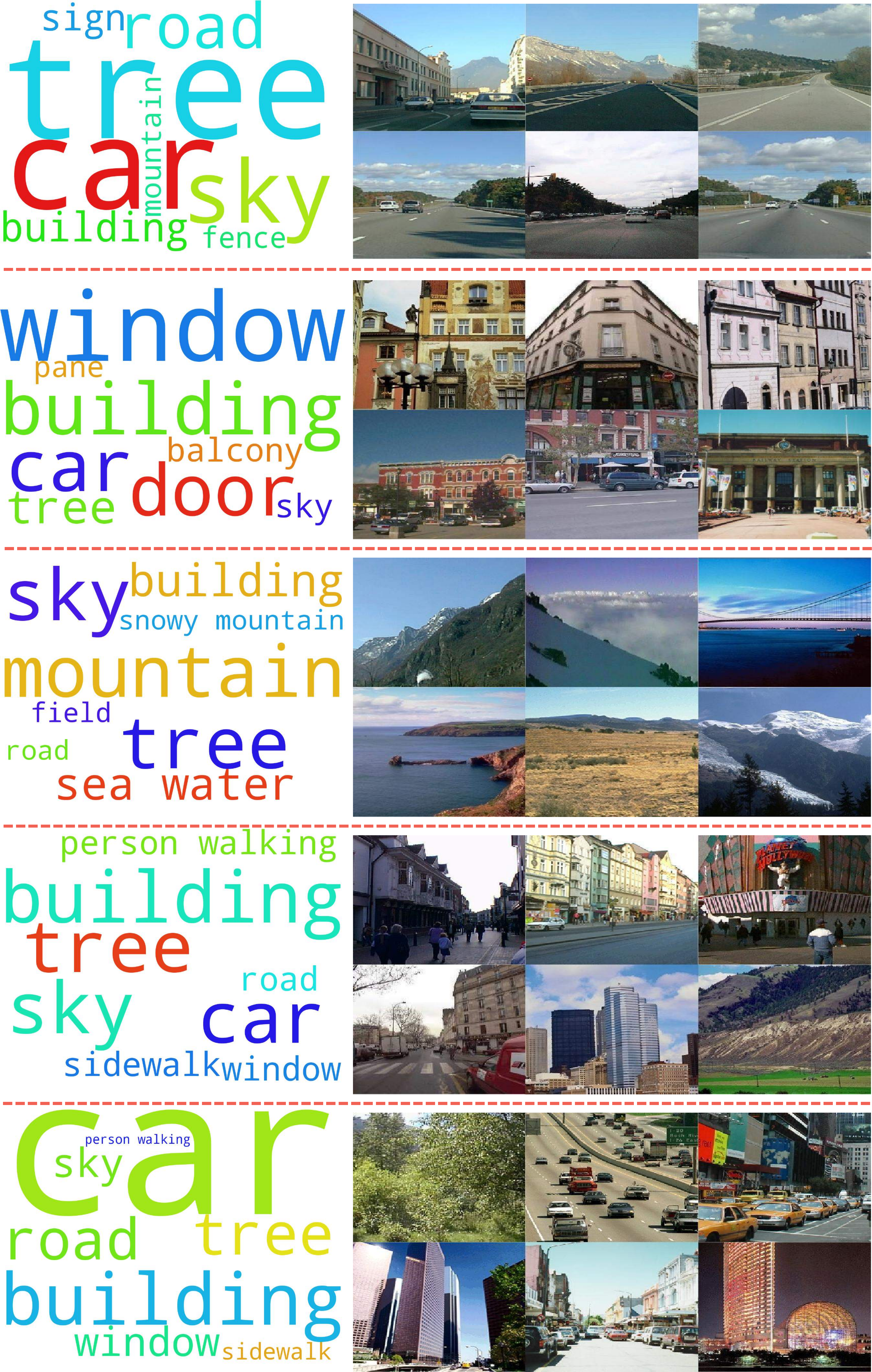}
\par\end{centering}

}\qquad{}\subfloat[
\label{fig:SL-graph}
]{\begin{centering}
\includegraphics[width=0.3\paperwidth]{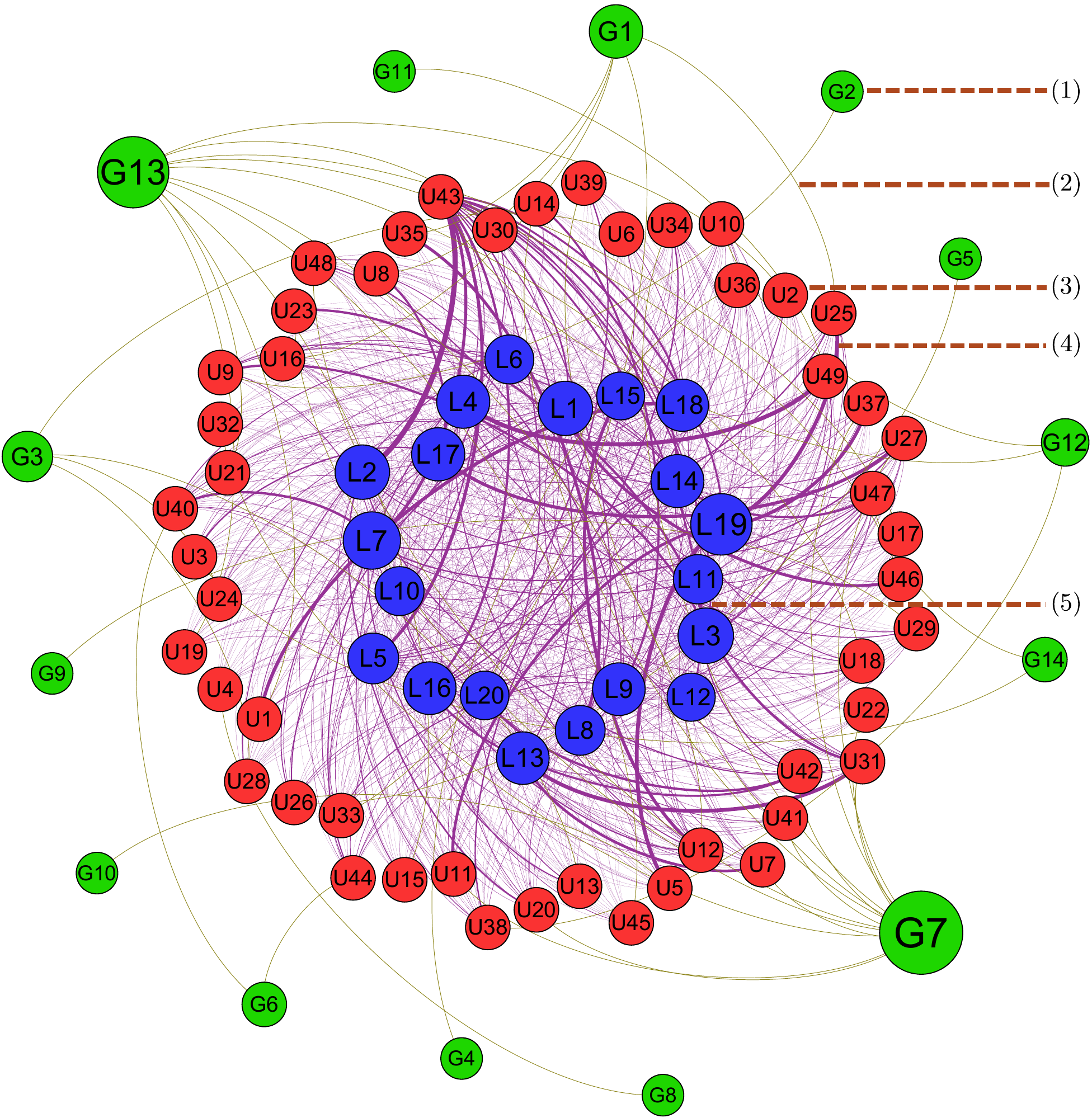}
\par\end{centering}

}
\par\end{centering}

\caption{
Clustering representation for two datasets: (a) Five image clusters
from \emph{Labelme} data discovered by MWMS algorithm: tag-clouds on the left are accumulated from all images in the clusters while six images on the right are randomly chosen images in that cluster; (b) StudentLife discovered network with three node groups: (1) discovered student clusters, (3) student nodes, (5) discovered activity location (from WiFi data); and two edge groups: (2) Student to cluster assignment, (4) Student involved to activity location. Node sizes (of discovered nodes) depict the number of element in clusters while edge sizes between \emph{Student} and \emph{activity location }represent the popularity of student's activities.
}
\end{figure*}

\begin{table}
\caption{Clustering performance for LabelMe dataset. MWM = Multilevel Wasserstein Means and MWMS = MWM with shared atoms on the measure space.}
\label{tab:clustering_metrics}
\centering
\begin{tabular}{lccc}
\toprule
Methods & NMI & ARI & AMI  \\
\hline
\midrule
K-means & 0.370 & 0.282 & 0.365 \\
TSK-means & 0.203 & 0.101 & 0.170\\
MC2 \citep{Viet-2016} & 0.315 & 0.206 & 0.273\\
\textbf{MWM} & 0.374 & 0.302 & 0.368  \\
\textbf{MWMS} &\textbf{ 0.416} & \textbf{0.355} & \textbf{0.411}  \\
\textbf{MWM with context} &0.662 &0.580& 0.654  \\
\textbf{MWMS with context} &\textbf{0.675} & \textbf{0.603} &\textbf{0.666}  \\
\bottomrule
\end{tabular}
\end{table}


\textbf{StudentLife dataset}\footnote{https://studentlife.cs.dartmouth.edu/dataset.html} is a large dataset frequently used in pervasive and ubiquitous computing research. Data signals consist of multiple channels  (e.g. WiFi signals and Bluetooth scan), which are collected from smartphones of 49 students at Dartmouth College over a 10-week spring term in 2013. However, in our experiments, we use only WiFi signal strengths. We apply a similar procedure described in \citep{nguyen_nguyen_venkatesh_phung_icpr16mcnc} to pre-process the
data. We aggregate the number of scans by each WiFi access point and select 500 WiFi IDs with the highest frequencies. Eventually, we obtain 49 ``documents'' with totally approximately $4.6$ million 500-dimensional data points.

\textbf{Clustering evaluation metrics:} We use normalized Mutual Information (NMI), adjusted Rand index (ARI), and adjusted Mutual Information (AMI) to evaluate the performance metrics of clustering tasks. NMI measures mutual information between ground truth labels of data and  predicted cluster labels normalized by the average entropy of these labels, $NMI(U,V)=\nicefrac{2*MI(U,V)}{(H(U)+H(V))}$ where $MI(.,.)$ and $H(.)$ are respectively mutual information and entropy. AMI is an adjustment of the mutual information score which is generally higher for two clusterings with a larger number of clusters $AMI(U,V)={\nicefrac  {MI(U,V)-E\{MI(U,V)\}}{\max {\{H(U),H(V)\}}-E\{MI(U,V)\}}}$ where $E\{MI(.,.)\}$ is the expect ion of mutual information. ARI is an adjusted version of Rand index which computes similarity two clusterings by considering all pairs of samples and counting pairs that are assigned in the same or different clusters \citep{hubert1985comparing}.

\textbf{Quantitative results:} To quantitatively evaluate our proposed methods, we compare our algorithms with several baseline methods:
K-means, 3-stage K-means (TSK-means) as described in Appendix~\ref{sec:appendix_d}, MC2-SVI without context \citep{Viet-2016}. Clustering performance in Table \ref{tab:clustering_metrics} is evaluated with the image clustering problem on \emph{LabelMe dataset}. With K-means, we average all data points to obtain a single vector for each image. K-means needs much less time to run since the number of data points is now reduced to $1,800$. For MC2-SVI, we used stochastic variational inference and  parallelized Spark-based implementation in \citep{Viet-2016} to carry out the experiments. This implementation has the advantage of making use of all of 16 cores on the test machine.  
In terms of clustering accuracy, MWMS and MWMS with context algorithms perform the best.

Figure \ref{fig:labelme_clusters} demonstrates five representative image clusters with six randomly chosen images in each (on the right) which are discovered by our MWMS algorithm. We also accumulate labeled tags from all images in each cluster to produce the tag cloud on the left, which can be considered as the visual ground truth of clusters. Our algorithm can group images into clusters that are consistent
with the tag cloud.

\textbf{Qualitative results:} We use the StudentLife dataset to demonstrate the capability of multilevel clustering with large-scale datasets. This dataset not only contains a large number of data points but presents in high dimension. Our algorithms need approximately 1 hour to perform multilevel
clustering on this dataset. Figure \ref{fig:SL-graph} presents two levels of clusters discovered by our algorithms. The innermost (blue) and outermost (green) rings depict local and global clusters respectively. Global clusters represent groups of students while local clusters shared between students (``documents'') may be used to infer locations of students' activities. From these clustering results, we can dissect
students' shared location (activities), e.g. Student 49 (\emph{U49}) mainly took part in activities at location 4 (\emph{L4}). 

\textbf{Robust multilevel clustering:}
We now conduct experiments to demonstrate how \emph{multilevel Wasserstein geometric median (MWGM)} algorithm can be robust with some proportions of ``noise'' data. We created their versions of ``noise'' LabelMe which we add into the original dataset three varying proportions of Gaussian noise including $0.5\%$, $1\%$, and $5\%$ respectively. We now apply two algorithms, MWM and MWGM, with the contaminated LabelMe dataset. Table \ref{tab:W1_W2_noisy_clustering_metrics} shows the clustering performance of these algorithms. The clustering performance of the MWGM algorithm is more robust to the level of noise data added in compared with its second order counterpart, the MWM. It is also interesting that the clustering performance of MWGM is increasing then decreasing when more noise presenting in the data.
\comment{
\scriptsize
\begin{tabular}{|c|c|c|c|c|c|c|c|c|c|c|c|c|}
\hline 
Noise percentage &\multicolumn{3}{c|}{$0\%$} & \multicolumn{3}{c|}{$0.5\%$} & \multicolumn{3}{c|}{$1\%$} & \multicolumn{3}{c|}{$5\%$}\tabularnewline
\hline 
\diagbox{\small{Methods}}{\small{Metrics}}& NMI & ARI & AMI & NMI & ARI & AMI & NMI & ARI & AMI & NMI & ARI & AMI\tabularnewline
\hline 
\textbf{MWGM}& NMI & ARI & AMI& 0.493&0.440   & 0.484 & 0.553 & 0.506 & 0.533 &0.542  &0.508   &0.525   \tabularnewline
\hline 
\textbf{MWM}& NMI & ARI & AMI &0.412 & 0.340 & 0.413 & 0.501 &0.461  &  0.486 &0.470   & 0.409 & 0.454 \tabularnewline
\hline 
\end{tabular}
}

\begin{table}
\begin{centering}
\begin{tabular}{cccc}
\hline 
\multirow{2}{*}{{\footnotesize{}Noise percentage}} & \multirow{2}{*}{Metrics} & \multicolumn{2}{c}{Methods}\tabularnewline
\cline{3-4} 
 &  & \textbf{MWGM} & \textbf{MWM}\tabularnewline
\hline 
\hline 
\multirow{3}{*}{$0\%$} & NMI & 0.536 & 0.473   \tabularnewline
 & ARI & 0.501 & 0.427  \tabularnewline
 & AMI & 0.530 & 0.465  \tabularnewline
\hline 
\multirow{3}{*}{$0.5\%$} & NMI & 0.493 & 0.412\tabularnewline
 & ARI & 0.440 & 0.340\tabularnewline
 & AMI & 0.484 & 0.413\tabularnewline
\hline 
\multirow{3}{*}{$1\%$} & NMI & 0.553 & 0.501\tabularnewline
 & ARI & 0.506 & 0.461\tabularnewline
 & AMI & 0.533 & 0.486\tabularnewline
\hline 
\multirow{3}{*}{$5\%$} & NMI & 0.542 & 0.470\tabularnewline
 & ARI & 0.508 & 0.409\tabularnewline
 & AMI & 0.525 & 0.454\tabularnewline
\hline 
\end{tabular}

\par\end{centering}

\caption{\label{tab:W1_W2_noisy_clustering_metrics}Clustering performance for LabelMe dataset with noisy data.}
\end{table}
\subsection{Wall-clock running time analysis}
In this section, we illustrate the running time of sequential and parallel implementations of the proposed algorithms on both synthetic and real-world datasets. For synthetic data, we use datasets with different numbers of data groups (i.e. 1000, 2000, 4000, 8000, 16,000, and 20,000). With each dataset, we run four algorithms with/without local constraint and with/without context observations on sequential and parallelized implementations. All experiments are conducted on the same machine (Windows 10 64-bit, core i7 3.4GHz CPU and 16GB RAM). We then observe the average running time of each iteration of serial and parallelized implementations of MWGM(S) (first order Wasserstein metric) and MWM(S) (second order Wasserstein metric) in  Figures \ref{fig:synthetic_running_time_W1 } and \ref{fig:synthetic_running_time_W2} respectively. The parallelized implementations have significantly reduced the wall-clock running time of the proposed algorithms especially when the number of groups is large. Since our experiments for parallelized implementations are conducted on a station machine with multiple processors, it is obvious the running time complexity will reduce more dramatically on cluster systems when the datasets contain an extremely large number of groups, e.g. millions.
\begin{figure*}[ht]
\centerline{\includegraphics[width=1 \textwidth]{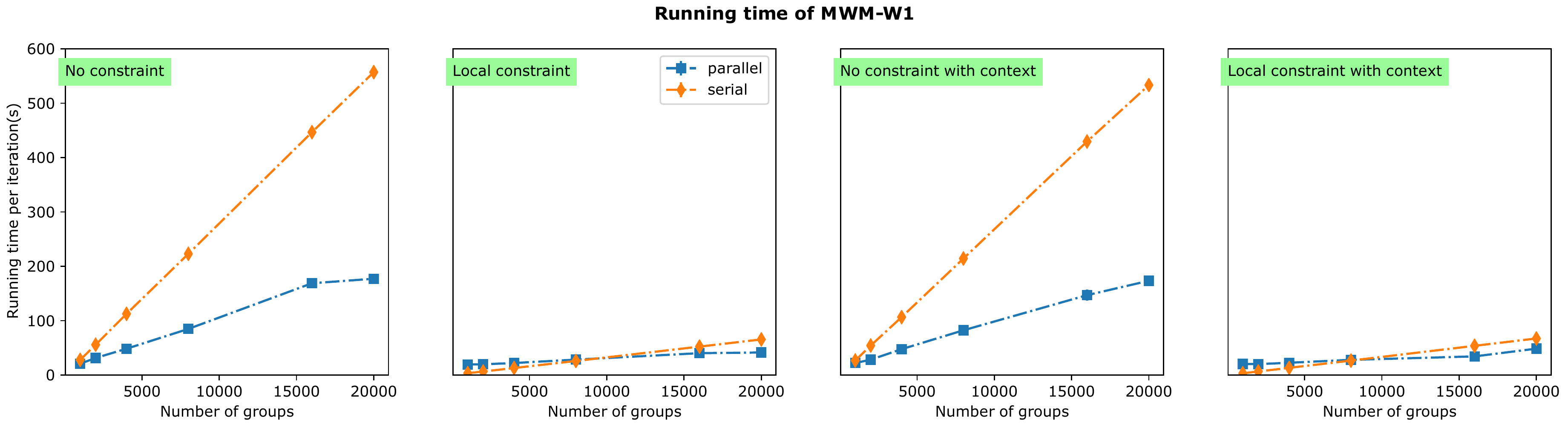}}
\caption{Comparison of running time between serial and parallelized implementations of Multilevel Wasserstein Geometric Median (MWGM) with 4 different settings: \textit{no constraint, local constraint, no constraint with context, local constraint with context}. 
}
\label{fig:synthetic_running_time_W1 }  
\end{figure*}
\begin{figure*}[ht]
\centerline{\includegraphics[width=1 \textwidth]{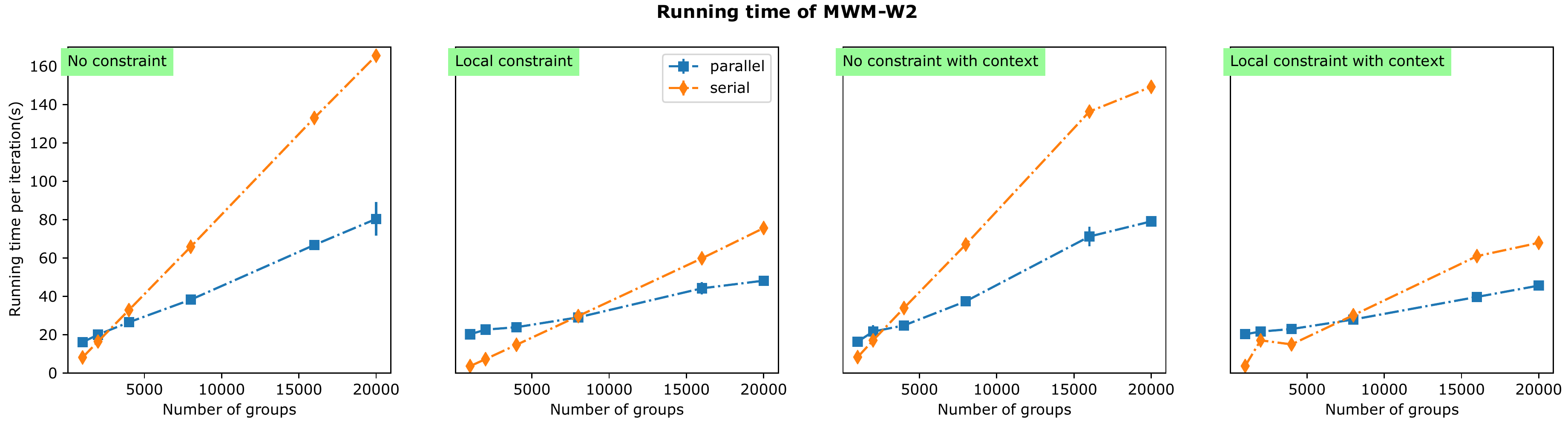}}
\caption{Comparison of running time between serial and parallelized implementations of Multilevel Wasserstein Mean (MWM) with 4 different settings: \textit{no constraint, local constraint, no constraint with context, local constraint with context}.
}
\label{fig:synthetic_running_time_W2}  
\end{figure*}
We now present the running time of the proposed algorithms on the real-world dataset LabelMe. Figure \ref{fig:LabelMe_running_time} depicts the running time for this dataset which shows the significant reduction in running time of the parallelized implementation compared with the serial version. Although there are less than $3000$ groups of data in this dataset, the running time with the parallelized version of our proposed algorithms has reduced approximately twice. It is also noted that the MWGM algorithm takes more time to run since its inner loop approximates the geometric median instead of the closed-form in the MWM algorithm.

\section{Conclusion} \label{Section:discussion}
We have proposed an optimization-based approach to multilevel clustering using Wasserstein metrics. There are several possible directions for extensions. Firstly, we have only considered continuous data. Hence, it is of interest to extend our formulation to discrete data. Secondly, our  method requires knowledge of the numbers of clusters both in local and global clustering.  When these numbers are unknown, it seems reasonable to incorporate a penalty on the model complexity.
\begin{figure*}[t!]
\centerline{\includegraphics[width=1 \textwidth]{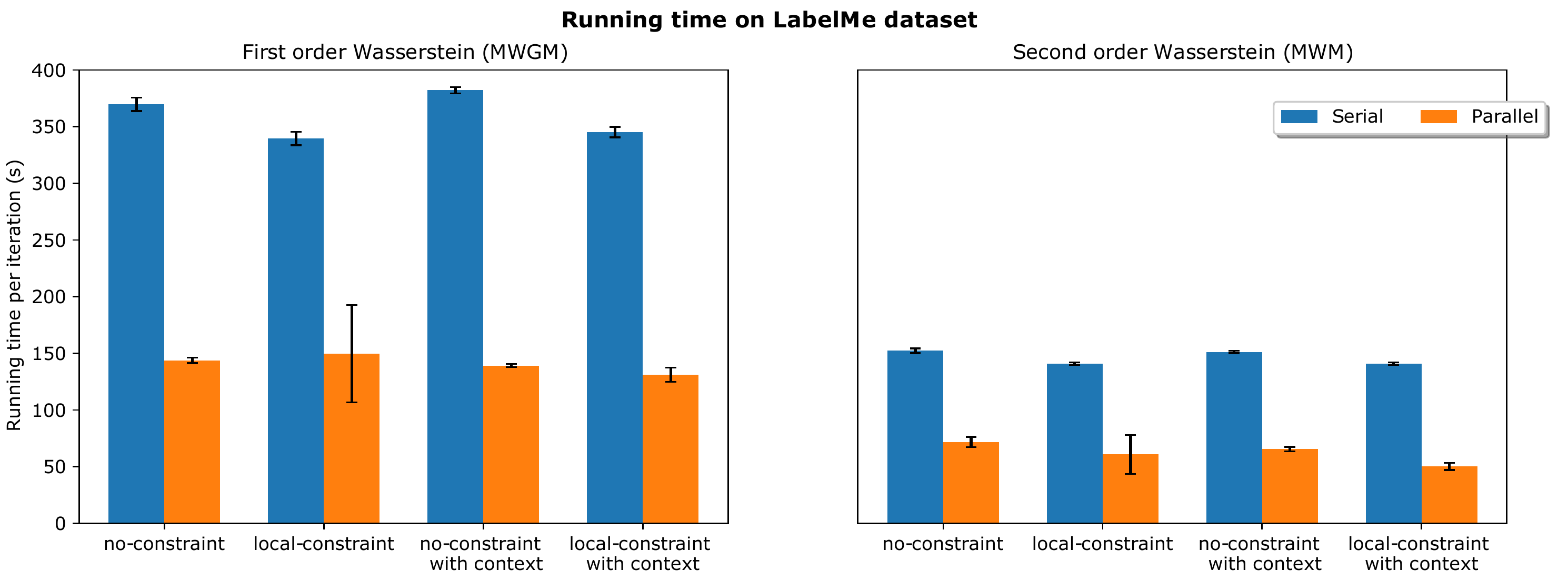}}
\caption{Comparison of running time between serial and parallelized implementations of MWM and MWGM with 4 different settings: \textit{no constraint, local constraint, no constraint with context, local constraint with context} on LabelMe dataset.}
\label{fig:LabelMe_running_time}  
\end{figure*}
\acks{This research is supported in part by grants NSF CAREER DMS-1351362, NSF CNS-1409303, the  Margaret  and  Herman  Sokol  Faculty  Award  and  research gift from Adobe Research (XN). DP, VH and ND gratefully acknowledge the support from the Australian Research Council (ARC) DP150100031 and DP160109394. }


\appendix
\section{Wasserstein barycenter}
\label{sec:appendix_a}
In this appendix, we collect relevant information on the Wasserstein metric and Wasserstein barycenter problem, which were introduced in Section \ref{Section:prelim} in the paper. For any Borel map $g: \Theta \to \Theta$ and probability measure $G$ on $\Theta$, the push-forward measure of $G$ through $g$, denoted by $g\#G$, is defined by the condition that $\int \limits_{\Theta}{f(y)}\mathrm{d}(g\#G)(y)=\int \limits_{\Theta}{f(g(x))}\mathrm{d}G(x)$ for every continuous bounded function $f$ on $\Theta$.

\paragraph{Wasserstein metric:} \label{Section:Append_Wasserstein_metric}
We now provide further discussion about the formulation of Wasserstein metric when the probability measures are discrete. In particular, when $G = \sum \limits_{i=1}^{k}{p_{i}\delta_{\theta_{i}}}$ and $G'=\sum  \limits_{i=1}^{k'}{p_{i}'\delta_{\theta_{i}'}}$ are discrete measures with finite support, i.e. $k$ and $k'$ are finite, the Wasserstein distance of order $r$ between $G$ and $G'$ can be represented as
\begin{eqnarray}
W_{r}^{r}(G,G') = \min \limits_{T \in \Pi(G,G')} \langle T,M_{G,G'} \rangle, \label{eqn:Wasserstein_computation}
\end{eqnarray} 
where we have
\begin{eqnarray}
\Pi(G,G') = \left\{T \in \mathbb{R}_{+}^{k \times k'}: T\mathbbm{1}_{k'}=\vec{p}, \ T\mathbbm{1}_{k}=\vec{p}'\right\} \nonumber
\end{eqnarray}
such that $\vec{p}=(p_{1},\ldots,p_{k})^{T}$ and $\vec{p}'=(p_{1}',\ldots,p_{k'}')^{T}$, 
$M_{G,G'} = \left\{\|\theta_{i}-\theta_{j}'\|^{r}\right\}_{i,j} \in \mathbb{R}_{+}^{k \times k'}
$ is the cost matrix, i.e. matrix of pairwise distances of elements between $G$ and $G'$, and 
$\langle A, B \rangle =  \text{tr}(A^{T}B)$ is the Frobenius dot-product of matrices. The optimal $T \in \Pi(G,G')$ in optimization problem \eqref{eqn:Wasserstein_computation} is called the optimal coupling of $G$ and $G'$, representing the optimal transport between these two measures. 
\paragraph{Wasserstein barycenter:} \label{Section:Append_Wasserstein_barycenter}
As introduced in Section \ref{Section:prelim} in the paper, for any probability measures $P_{1}, P_{2}, \ldots, P_{N} \in \mathcal{P}_{2}(\Theta)$, their Wasserstein barycenter $\overline{P}_{N,\lambda}$ is such that
\begin{eqnarray}
\overline{P}_{N,\lambda}=\mathop {\arg \min}\limits_{P \in \mathcal{P}_{2}(\Theta)}{\sum \limits_{i=1}^{N}{\lambda_{i}W_{2}^{2}(P,P_{i})}}, \nonumber
\end{eqnarray} 
where $\lambda \in \Delta_{N}$ denote weights associated with $P_{1},\ldots,P_{N}$. According to \citep{Carlier-2011}, $P_{N,\lambda}$ can be obtained as a solution to  so-called multi-marginal optimal transportation problem. In fact, if we denote $T_{k}^{1}$ as the measure preserving map from $P_{1}$ to $P_{k}$, i.e. 
$P_{k}=T_{k}^{1} \# P_{1}$, for any $1 \leq k \leq N$ (note that, these maps exist as long as we assume for example that the probability measures $P_{1}, \ldots, P_{N}$ have density functions~\citep{Santambrogio-2015}), then \begin{eqnarray}
\overline{P}_{N,\lambda}=\biggr(\sum \limits_{k=1}^{N}{\lambda_{k}T_{k}^{1}}\biggr)\# P_{1}. \nonumber
\end{eqnarray}
Unfortunately, the forms of the maps $T_{k}^{1}$ are analytically intractable, especially if no special constraints on $P_{1}, \ldots, P_{N}$ are imposed.

Recently,~\cite{Anderes-2015} studied the Wasserstein barycenters $\overline{P}_{N,\lambda}$ when the probability measures $P_{1}, P_{2}, \ldots, P_{N}$ are finite discrete and $\lambda=\biggr(1/N,\ldots,1/N\biggr)$. They demonstrate the following sharp result (cf. Theorem 2 
in \citep{Anderes-2015}) regarding the number of atoms of $\overline{P}_{N,\lambda}$
\begin{customthm}{A.1} \label{theorem:upperbound_barycenter} There exists a Wasserstein 
barycenter $\overline{P}_{N,\lambda}$ such that $|\text{supp}(\overline{P}_{N,\lambda})| \leq \sum \limits_{i=1}^{N}{s_{i}}-N+1$.
\end{customthm}
Therefore, when $P_{1},\ldots, P_{N}$ are indeed finite discrete measures and the weights are uniform, the problem of finding Wasserstein barycenter $\overline{P}_{N,\lambda}$ over the (computationally large) space $\mathcal{P}_{2}(\Theta)$ is reduced to a search over a smaller space $\mathcal{O}_{l}(\Theta)$ where $l=\sum \limits_{i=1}^{N}{s_{i}-N+1}$.

\section{Proofs of theorems}
\label{sec:appendix_b}
In this appendix, we provide proofs for the remaining results in the paper. We start by giving a proof for the transition from multilevel Wasserstein means objective function \eqref{eqn:multilevel_Kmeans_typeone} to objective function \eqref{eqn:multilevel_K_means_typeone_first} in Section \ref{Section:multilevel_kmeans} in the paper. All the notations in this appendix are similar to those in the main text. For each closed subset $\mathcal{S} \subset \mathcal{P}_{2}(\Theta)$, 
we denote the Voronoi region generated by $\mathcal{S}$ on the space $\mathcal{P}_{2}(\Theta)$ by the collection of subsets $\{ V_P \}_{P \in \mathcal{S}}$, 
where $V_P := \{Q \in \mathcal{P}_{2}(\Theta) : W_{2}^{2}(Q,P) = 
\mathop {\min }\limits_{G \in \mathcal{S}} W_{2}^{2}(Q,G)\}$. We define the projection mapping $\pi _\mathcal{S}$ as: $\pi _{\mathcal{S}} :\mathcal{P}_{2}(\Theta)  \to \mathcal{S}$ 
where $\pi _{\mathcal{S}} (Q) = P$ as $Q \in V_{P}$. Note that, for any $P_{1}, P_{2} \in \mathcal{S}$ such that $V_{P_{1}}$ and $V_{P_{2}}$ share the boundary, the values of $\pi_{S}$ at the elements in that boundary can be chosen to be either $P_{1}$ or $P_{2}$. Now, we start with the following useful lemmas.

\begin{customlem}{B.1} \label{lemma:one} For any closed subset $\mathcal{S}$ on $\mathcal{P}_{2}(\Theta)$, if $\mathcal{Q} \in \mathcal{P}_{2}(\mathcal{P}_{2}(\Theta))$, then $E_{X \sim \mathcal{Q}} (d_{W_{2}}^{2}(X,\mathcal{S})) = W_2^2 (\mathcal{Q},\pi _\mathcal{S} \# \mathcal{Q})$ where $d_{W_{2}}^{2}(X,\mathcal{S})=\inf \limits_{P \in \mathcal{S}}{W_{2}^{2}(X,P)}$.
\end{customlem}
\begin{proof}
For any element $\pi \in \Pi (\mathcal{Q},\pi _\mathcal{S} \# \mathcal{Q})$:
\begin{eqnarray}
\int{W_{2}^{2}(P,G)} \mathrm{d}\pi (P,G) & \geq &  \int {d_{W_{2}}^{2}(P,\mathcal{S})} \mathrm{d}\pi (P,G) \nonumber \\
& = & \int {d_{W_{2}}^{2}(P,\mathcal{S})} \mathrm{d}\mathcal{Q}(P)
= E_{X \sim \mathcal{Q}} (d_{W_{2}}^{2}(X,\mathcal{S})), \nonumber
\end{eqnarray}
where the integrations in the first two terms range over $\mathcal{P}_{2}(\Theta)  \times \mathcal{S}$ while that in the final term ranges over $\mathcal{P}_{2}(\Theta)$. Therefore, we obtain 
\begin{eqnarray}
W_2^2 (\mathcal{Q},\pi _\mathcal{S} \# \mathcal{Q}) & = & \mathop {\inf } \int\limits_{\mathcal{P}_{2}(\Theta)  \times \mathcal{S}} {W_{2}^{2}(P,G)} \mathrm{d}\pi (P,G)
\geq E_{X \sim \mathcal{Q}} (d_{W_{2}}^{2}(X,\mathcal{S})), \label{eqn:lemmaequationone}
\end{eqnarray}
where the infimum in the first equality ranges over all $\pi \in \Pi (\mathcal{Q},\pi _\mathcal{S} \# \mathcal{Q})$. 

On the other hand, let ${\displaystyle g:\mathcal{P}_{2}(\Theta)  \to \mathcal{P}_{2}(\Theta) \times \mathcal{S}}$ 
such that $g(P)=(P,\pi_{\mathcal{S}}(P))$ for all $P \in \mathcal{P}_{2}(\Theta)$. Additionally, let 
$\mu _{\pi _\mathcal{S} } = g \# \mathcal{Q}$, the push-forward measure of $\mathcal{Q}$ under mapping $g$. It is clear that $\mu _{\pi _\mathcal{S}}$ is a coupling between $\mathcal{Q}$ and $\pi _\mathcal{S} \# \mathcal{Q}$. Under this construction, we obtain for any $X \sim \mathcal{Q}$ that
\begin{eqnarray}
 E\left(W_{2}^{2}(X,\pi _\mathcal{S} (X))\right) & = & \int {W_{2}^{2}(P,G) } \mathrm{d}\mu _{\pi _\mathcal{S}} (P,G) \nonumber \\
 & \geq & \mathop {\inf } \int {W_{2}^{2}(P,G)} \mathrm{d}\pi (P,G)
= W_2^2 (\mathcal{Q},\pi _\mathcal{S} \# \mathcal{Q}), \label{eqn:lemmaequationtwo}
\end{eqnarray}
where the infimum in the second inequality ranges over all $\pi  \in \Pi (\mathcal{Q},\pi _\mathcal{S} \# \mathcal{Q})$ and the integrations range over $\mathcal{P}_{2}(\Theta) \times \mathcal{S}$. Now, from the definition of $\pi_{\mathcal{S}}$
\begin{eqnarray}
E(W_{2}^{2}(X,\pi _\mathcal{S} (X))) & = & \int {W_{2}^{2}(P,\pi _\mathcal{S}(P))}\mathrm{d}\mathcal{Q}(P) \nonumber \\
& = & \int {d_{W_{2}}^{2}(P,\mathcal{S})} \mathrm{d}\mathcal{Q}(P)
= E(d_{W_{2}}^{2}(X,\mathcal{S})), \label{eqn:lemmaequationthree}
\end{eqnarray}
where the integrations in the above equations range over $\mathcal{P}_{2}(\Theta)$. By combining \eqref{eqn:lemmaequationtwo} and \eqref{eqn:lemmaequationthree}, we would obtain that
\begin{eqnarray}
E_{X \sim \mathcal{Q}} (d^{2}_{W_{2}}(X,\mathcal{S})) \ge W_2^2 (\mathcal{Q},\pi _\mathcal{S} \# \mathcal{Q}). \label{eqn:lemmaequationfourth}
\end{eqnarray}
From \eqref{eqn:lemmaequationone} and \eqref{eqn:lemmaequationfourth}, it is straightforward that $E_{X \sim Q} (d(X,S)^2 ) = W_2^2 (Q,\pi _S  \# Q)$. Therefore, we achieve the conclusion of the lemma.
\end{proof}

\begin{customlem}{B.2} \label{lemma:two}
For any closed subset $\mathcal{S} \subset \mathcal{P}_{2}(\Theta)$ and $\mu \in \mathcal{P}_{2}(\mathcal{P}_{2}(\Theta))$ 
with $\text{supp}(\mu) \subseteq \mathcal{S}$, 
there holds 
$W_2^2 (\mathcal{Q},\mu ) \ge W_2^2 (\mathcal{Q},\pi _\mathcal{S} \# \mathcal{Q})$ for any $\mathcal{Q} \in \mathcal{P}_{2}(\mathcal{P}_{2}(\Theta))$.
\end{customlem}
\begin{proof}
Since $\text{supp}(\mu) \subseteq \mathcal{S}$, it is clear that $W_2^2 (\mathcal{Q},\mu) = {\displaystyle \mathop {\inf }\limits_{\pi  \in \Pi (\mathcal{Q},\mu )} \int\limits_{\mathcal{P}_{2}(\Theta) \times \mathcal{S}} {W_{2}^{2}(P,G)} \mathrm{d}\pi (P,G)}$.\\
Additionally, we have
\begin{eqnarray}
\int {W_{2}^{2}(P,G)} \mathrm{d}\pi (P,G) & \geq & \int {d_{W_{2}}^{2}(P,\mathcal{S})} \mathrm{d}\pi (P,G)
= \int {d_{W_{2}}^{2}(P,\mathcal{S})} \mathrm{d}\mathcal{Q}(P) \nonumber \\
& = & E_{X \sim Q} (d_{W_{2}}^{2}(X,S))
= W_2^2 (\mathcal{Q},\pi _\mathcal{S} \# \mathcal{Q}), \nonumber
\end{eqnarray}
where the last inequality is due to Lemma \ref{lemma:one} and the integrations in the first two terms range over $\mathcal{P}_{2}(\Theta) \times \mathcal{S}$ while that in the final term ranges over $\mathcal{P}_{2}(\Theta)$. Therefore, we achieve the conclusion of the lemma.
\end{proof}
Equipped with Lemma \ref{lemma:one} and Lemma \ref{lemma:two}, 
we are ready to establish the equivalence between multilevel Wasserstein means objective function  \eqref{eqn:multilevel_K_means_typeone_first} and objective function  \eqref{eqn:multilevel_Kmeans_typeone} in Section \ref{Section:multilevel_kmeans} in the main text.
\begin{customlem}{B.3} \label{proposition:Wassersteinequivalence}
For any given positive integers $m$ and $M$, we have
\begin{eqnarray}
A : = \inf \limits_{\Hcal \in \mathcal{E}_{M}(\mathcal{P}_{2}(\Theta))} W_{2}^{2}(\Hcal,\dfrac{1}{m}\mathop {\sum }\limits_{j=1}^{m}{\delta_{G_{j}}})
= \dfrac{1}{m}\inf \limits_{\Hbold = (H_{1},\ldots,H_{M})}\sum \limits_{j=1}^{m} d_{W_{2}}^{2}(G_{j},\Hbold) := B. \nonumber
\end{eqnarray}
\end{customlem}
\begin{proof}
Write $\mathcal{Q}=\dfrac{1}{m}\mathop {\sum }\limits_{j=1}^{m}{\delta_{G_{j}}}$. From the definition of $B$, for any $\epsilon>0$, we can find $\overline{\Hbold}$ such that 
\begin{eqnarray}
B \geq \dfrac{1}{m}\sum \limits_{j=1}^{m} d_{W_{2}}^{2}(G_{j},\overline{\Hbold}) - \epsilon
= E_{X \sim \mathcal{Q}}(d_{W_{2}}^{2}(X,\overline{\Hbold})) - \epsilon
= W_{2}^{2}(\mathcal{Q},\pi_{\overline{\Hbold}} \# \mathcal{Q}) {\bf - \epsilon}
\geq A -\epsilon, \nonumber
\end{eqnarray}
where the second equality in the above display
is due to Lemma \ref{lemma:one} while the last 
inequality is from the fact that $\pi_{\overline{\Hbold}} \# \mathcal{Q}$ is a discrete probability measure in $\mathcal{P}_{2}(\mathcal{P}_{2}(\Theta))$ with exactly $M$ support points. Since the inequality in the above display holds for any $\epsilon$, it implies that $B \geq A$. On the other hand, from the formation of $A$, for any $\epsilon>0$, we also can find $\Hcal' \in \mathcal{E}_{M}(\mathcal{P}_{2}(\Theta))$ such that
\begin{eqnarray}
A \geq W_{2}^{2}(\Hcal',\mathcal{Q}) - \epsilon
\geq W_{2}^{2}(\mathcal{Q},\pi_{\Hbold'} \# \mathcal{Q}) - \epsilon
= \dfrac{1}{m}\sum \limits_{j=1}^{m} d_{W_{2}}^{2}(G_{j},\Hbold') -\epsilon
\geq B - \epsilon, \nonumber
\end{eqnarray}
where  $\Hbold' = \text{supp}(\Hcal')$, the second inequality is due to Lemma \ref{lemma:two}, and the third equality is due to Lemma \ref{lemma:one}. Therefore, it means that $A \geq B$. We achieve the conclusion of the lemma. 
\end{proof}
\begin{customprop}{B.4} \label{lemma:equivalence_multilevels_Kmeans}
For any positive integer numbers $m, M$ and $k_{j}$ as $1 \leq j \leq m$, we denote 
\begin{eqnarray}
C & : = & \mathop {\inf }\limits_{\substack {G_{j} \in \mathcal{O}_{k_{j}}(\Theta) \ \forall 1 \leq j \leq m, \\ \Hcal \in \mathcal{E}_{M}(\mathcal{P}_{2}(\Theta))}}{\mathop {\sum }\limits_{i=1}^{m}{W_{2}^{2}(G_{j},P_{n_{j}}^{j})}}
+ \lambda W_{2}^{2}(\Hcal,\dfrac{1}{m}\mathop {\sum }\limits_{i=1}^{m}{\delta_{G_{i}}}) \nonumber \\
D & : = & \mathop {\inf }\limits_{\substack {G_{j} \in \mathcal{O}_{k_{j}}(\Theta) \ \forall 1 \leq j \leq m, \\ \Hbold = (H_{1},\ldots,H_{M})}}{\mathop {\sum }\limits_{j=1}^{m}{W_{2}^{2}(G_{j},P_{n_{j}}^{j})}}
+ \dfrac{\lambda \ d_{W_{2}}^{2}(G_{j},\Hbold)}{m}. \nonumber
\end{eqnarray}
Then, we have $C = D$.
\end{customprop}
\begin{proof} The proof of this proposition is a straightforward application of Lemma \ref{proposition:Wassersteinequivalence}. Indeed, for each fixed $(G_{1},\ldots,G_{m})$ the infimum w.r.t to $\mathcal{H}$ in $C$ leads to the same infimum w.r.t to $\Hbold$ in $D$, according to Lemma \ref{proposition:Wassersteinequivalence}. Now, by taking the infimum w.r.t to $(G_{1},\ldots,G_{m})$ on both sides, we achieve the conclusion of the proposition.
\end{proof}

In the remainder of this Appendix, we present the proofs for all remaining
theorems stated in the main text.
\paragraph{PROOF OF THEOREM \ref{theorem:local_convergence_multilevel_Kmeans}} 
Recall that, we use the notation $\hat W_{2}$ to denote the entropic regularized second-order Wasserstein with some given regularized parameter (see equation~\eqref{eq:entropic_Wasserstein} for the details). Now, for any $G_{j} \in \mathcal{E}_{k_{j}}(\Theta)$ and $\Hbold =(H_{1},\ldots,H_{M})$, we denote the function 
\begin{eqnarray}
f(\vec{G}, \Hbold)=\mathop {\sum }\limits_{j=1}^{m}{\hat W_{2}^{2}(G_{j},P_{n}^{j})}+\dfrac{\lambda \ d_{\hat W_{2}}^{2}(G_{j},\Hbold)}{m}, \nonumber
\end{eqnarray}
where $\vec{G}=(G_{1},\ldots,G_{m})$. To obtain the conclusion of this theorem, it is sufficient to demonstrate for any $t \geq 0$ that
\begin{eqnarray}
f(\vec{G}^{(t+1)},\Hbold^{(t+1)}) < f(\vec{G}^{(t)},\Hbold^{(t)}) \nonumber
\end{eqnarray}
unless $(\vec{G}^{(t+1)},\Hbold^{(t+1)}) \equiv (\vec{G}^{(t)},\Hbold^{(t)})$. It is clear that $f(\vec{G}^{(t+1)},\Hbold^{(t+1)}) = f(\vec{G}^{(t)},\Hbold^{(t)})$ when $(\vec{G}^{(t+1)},\Hbold^{(t+1)}) \equiv (\vec{G}^{(t)},\Hbold^{(t)})$. Therefore, we will only consider the setting when $(\vec{G}^{(t+1)},\Hbold^{(t+1)}) \neq (\vec{G}^{(t)},\Hbold^{(t)})$. There are two cases: $\vec{G}^{(t+1)} \neq  \vec{G}^{(t)}$ or $\Hbold^{(t + 1)} \neq \Hbold^{(t)}$. 

\vspace{0.5 em}
\noindent
\textbf{Case 1:} We first consider the setting when $\vec{G}^{(t+1)} \neq  \vec{G}^{(t)}$. It means that there exists $j \in \{1, 2, \ldots, m\}$ such that $G_{j}^{(t+1)} \neq G_{j}^{(t)}$. Without loss of generality, we assume that $j = 1$. We show that $f(\vec{G}^{(t+1)},\Hbold^{(t)})  < f(\vec{G}^{(t)},\Hbold^{(t)})$ when $\vec{G}^{(t+1)} \neq  \vec{G}^{(t)}$. From the updates of $\vec{G}^{(t+1)}$ in Algorithm~\ref{alg:multilevels_Wasserstein_means}, for any $1 \leq j \leq m$ we have
\begin{align*}
G_{j}^{(t+1)} \in \mathop {\arg \min }\limits_{G_{j} \in \mathcal{O}_{k_{j}}(\Theta)}{\hat W_{2}^{2}(G_{j},P_{n_{j}}^{j})}+ \lambda \hat W_{2}^{2}(G_{j},H_{i_{j}}^{(t)})/m
\end{align*}
for all $1 \leq j \leq m$. It demonstrates that
\begin{align*}
	\hat W_{2}^{2}(G_{j}^{(t+1)},P_{n_{j}}^{j})+ \frac{\lambda}{m} \hat W_{2}^{2}(G_{j}^{(t+1)},H_{i_{j}}^{(t)}) & \\
	& \hspace{- 4 em} \leq \min_{G_{j} \in  \mathcal{O}_{k_{j}}(\Theta): \text{supp}(G_{j}) \in \text{supp}(G_{j}^{(t)})} \hat W_{2}^{2}(G_{j},P_{n_{j}}^{j})+ \frac{\lambda}{m} \hat W_{2}^{2}(G_{j},H_{i_{j}}^{(t)}).
\end{align*}
The RHS of the above inequality means that we only find the optimal measure $\bar{G}_{j} \in  \mathcal{O}_{k_{j}}(\Theta)$ such that its supports lie in the set of supports of $G_{j}^{(t)}$. Therefore, finding the optimal measure $\bar{G}_{j}$ of that objective function is equivalent to find its optimal masses, which is a strongly convex problem. It proves that for any $1 \leq j \leq m$
\begin{align*}
	\hat W_{2}^{2}(G_{j}^{(t+1)},P_{n_{j}}^{j})+ \frac{\lambda}{m} \hat W_{2}^{2}(G_{j}^{(t+1)},H_{i_{j}}^{(t)}) \leq \hat W_{2}^{2}(G_{j}^{(t)},P_{n_{j}}^{j})+ \frac{\lambda}{m} \hat W_{2}^{2}(G_{j}^{(t)},H_{i_{j}}^{(t)})
\end{align*}
and the equality only holds when $G_{j}^{(t+1)} \equiv G_{j}^{(t)}$. Since $G_{1}^{(t+1)} \neq G_{1}^{(t)}$, we have
\begin{align*}
	\hat W_{2}^{2}(G_{j}^{(t+1)},P_{n_{j}}^{j})+ \frac{\lambda}{m} \hat W_{2}^{2}(G_{j}^{(t+1)},H_{i_{j}}^{(t)}) < \hat W_{2}^{2}(G_{j}^{(t)},P_{n_{j}}^{j})+ \frac{\lambda}{m} \hat W_{2}^{2}(G_{j}^{(t)},H_{i_{j}}^{(t)}).
\end{align*}
Putting the above results together, we obtain $f(\vec{G}^{(t+1)},\Hbold^{(t)})  < f(\vec{G}^{(t)},\Hbold^{(t)})$ when $\vec{G}^{(t+1)} \neq  \vec{G}^{(t)}$. 

\vspace{0.5 em}
\noindent
\textbf{Case 2:} We now consider the setting when $\Hbold^{(t+1)} \neq  \Hbold^{(t)}$. Without loss of generality, we assume that $H_{1}^{(t+1)} \neq H_{1}^{(t)}$. We now prove that $f(\vec{G}^{(t+1)},\Hbold^{(t+1)}) < f(\vec{G}^{(t+1)},\Hbold^{(t)})$. Indeed, recall that $C_{i} = \left\{l: i_{l}=i\right\}$ for $1 \leq i \leq M$ where $i_{j} = \mathop {\arg \min}\limits_{1 \leq u \leq M}{\hat W_{2}^{2}(G_{j}^{(t+1)},H_{u}^{(t)})}$ and
\begin{align*}
H_{i}^{(t+1)} \in \mathop {\arg \min }\limits_{H_{i} \in \mathcal{P}_{2}(\Theta)}{\sum \limits_{l \in C_{i}}{\hat W_{2}^2(H_{i}, G_{l}^{(t+1)})}}.
\end{align*}
We denote by $\mathcal{A}_{i}$ the set of all supports of $H_{i}^{(t)}$ for $l \in C_{i}$ for any $1 \leq i \leq M$. From the formulation of $H_{i}^{(t+1)}$, we have 
\begin{align*}
	\sum \limits_{l \in C_{i}} \hat{W}_{2}^2(H_{i}^{(t+1)}, G_{l}^{(t+1)}) \leq \min \limits_{H_{i}: \ \text{supp}(H_{i}) \equiv  \mathcal{A}_{i}} {\sum \limits_{l \in C_{i}}{\hat W_{2}^2(H_{i}, G_{l}^{(t+1)})}}.
\end{align*}
The RHS objective function means that we only search for the optimal measure $\bar{H}_{i}$ that has supports in the set $\mathcal{A}_{i}$. It is equivalent to search for the optimal masses of $\bar{H}_{i}$ as the supports of $\bar{H}_{i}$ are fixed. The RHS objective function is strongly convex with respect to the masses of $\bar{H}_{i}$. It demonstrates that for any $i \in \{1, 2, \ldots, M\}$
\begin{align*}
	\sum \limits_{l \in C_{i}} \hat{W}_{2}^2(H_{i}^{(t+1)}, G_{l}^{(t+1)}) \leq \sum \limits_{l \in C_{i}}{\hat W_{2}^2(H_{i}^{(t)}, G_{l}^{(t+1)})},
\end{align*}
and the equality only holds when $H_{i}^{(t+1)} \equiv H_{i}^{(t)}$. Since $H_{1}^{(t+1)} \neq H_{1}^{(t)}$, we have
\begin{align*}
	\sum \limits_{l \in C_{i}} \hat{W}_{2}^2(H_{1}^{(t+1)}, G_{l}^{(t+1)}) < \sum \limits_{l \in C_{i}}{\hat W_{2}^2(H_{1}^{(t)}, G_{l}^{(t+1)})}.
\end{align*}
Collecting the above results leads to $f(\vec{G}^{(t+1)},\Hbold^{(t+1)}) < f(\vec{G}^{(t+1)},\Hbold^{(t)})$ when $\Hbold^{(t+1)} \neq  \Hbold^{(t)}$.

In summary, the results from Cases 1 and 2 prove that $f(\vec{G}^{(t+1)},\Hbold^{(t+1)}) < f(\vec{G}^{(t)},\Hbold^{(t)})$ unless $(\vec{G}^{(t+1)},\Hbold^{(t+1)}) \equiv (\vec{G}^{(t)},\Hbold^{(t)})$ for any $t \geq 0$. Furthermore, the argument of these cases suggests that if there exists $\bar{t}$ such that $f(\vec{G}^{(\bar{t}+1)},\Hbold^{(\bar{t}+1)}) = f(\vec{G}^{(\bar{t})},\Hbold^{(\bar{t})})$, then $(\vec{G}^{(t+1)},\Hbold^{(t+1)}) \equiv (\vec{G}^{(t)},\Hbold^{(t)})$ for all $t \geq \bar{t}$. It suggests that there exists $(\bar{\vec{G}}, \bar{\Hbold})$ such that $f(\vec{G}^{(t)},\Hbold^{(t)})$ converges to $f(\bar{\vec{G}}, \bar{\Hbold})$ and $(\bar{\vec{G}}, \bar{\Hbold})$ satisfies that
\begin{align*}
	\bar{G}_{j} & \in \mathop {\arg \min } \limits_{G_{j} \in \mathcal{O}_{k_{j}}(\Theta)}{\hat W_{2}^{2}(G_{j},P_{n_{j}}^{j})}+ \lambda \hat W_{2}^{2}(G_{j},\bar{H}_{\bar{i}_{j}})/m, \\
	\bar{H}_{i} & \in \mathop {\arg \min }\limits_{H_{i} \in \mathcal{P}_{2}(\Theta)}{\sum \limits_{l \in \bar{C}_{i}}{\hat W_{2}^2(H_{i}, \bar{G}_{l})}},
\end{align*}
where $\bar{i}_{j} =  \mathop {\arg \min} \limits_{1 \leq u \leq M}{\hat W_{2}^{2}(\bar{G}_{j},\bar{H}_{u})}$ for any $1 \leq j \leq m$ and $\bar{C}_{i} = \left\{l: \bar{i}_{l}=i\right\}$ for $1 \leq i \leq M$. As a consequence, we obtain the conclusion of the theorem.
\paragraph{PROOF OF THEOREM \ref{theorem:local_convergence_local_constraint_Kmeans}} 
The proof is quite similar to the proof of Theorem \ref{theorem:local_convergence_multilevel_Kmeans}. In fact, recall from the proof of Theorem \ref{theorem:local_convergence_multilevel_Kmeans} that for any $G_{j} \in \mathcal{E}_{k_{j}}(\Theta)$ and $\Hbold =(H_{1},\ldots,H_{M})$ we denote the function 
\begin{eqnarray}
f(\vec{G}, \Hbold)=\mathop {\sum }\limits_{j=1}^{m}{\hat{W}_{2}^{2}(G_{j},P_{n}^{j})}+\dfrac{\lambda}{m} d_{\hat{W}_{2}}^{2}(G_{j},\Hbold), \nonumber
\end{eqnarray}
where $\vec{G}=(G_{1},\ldots,G_{m})$. Now it is sufficient to demonstrate for any $t \geq 0$ that
\begin{eqnarray}
f(\vec{G}^{(t+1)},\Hbold^{(t+1)}) < f(\vec{G}^{(t)},\Hbold^{(t)}), \nonumber
\end{eqnarray}
unless $(S_{K}^{(t+1)}, \vec{G}^{(t+1)}, \Hbold^{(t+1)}) \equiv (S_{K}^{(t)}, \vec{G}^{(t)}, \Hbold^{(t)})$. Here, we only give the proof for that inequality under the setting when $S_{K}^{(t+1)} \neq S_{K}^{(t)}$ as the the proof argument for that inequality when $S_{K}^{(t+1)} \equiv S_{K}^{(t)}$ and $(\vec{G}^{(t+1)}, \Hbold^{(t+1)})) \neq  (\vec{G}^{(t)}, \Hbold^{(t)})$ is similar to the proof argument of Theorem~\ref{theorem:local_convergence_multilevel_Kmeans}. 
Indeed, by the definition of Wasserstein distances, we have
\begin{eqnarray}
E = m \cdot f(\vec{G}^{(t)},\Hbold^{(t)})  = \sum \limits_{u=1}^{m}{\sum \limits_{j,v}{mT_{j,v}^{u}\|a_{j}^{(t)}-X_{u,v}\|^{2}} +  \lambda U_{j,v}^{u}\|a_{j}^{(t)}-h_{i_{u},v}^{(t)}\|^{2}}. \nonumber
\end{eqnarray}
Therefore, the update of $S_{K}^{(t+1)}$ from Algorithm \ref{alg:local_constraint_multilevels_Wasserstein_means} and the assumption that $S_{K}^{(t+1)} \neq S_{K}^{(t)}$ leads to
\begin{eqnarray}
E & > &  \sum \limits_{u=1}^{m}{\sum \limits_{j,v}{mT_{j,v}^{u}\|a_{j}^{(t+1)}-X_{u,v}\|^{2}}}
+ \lambda U_{j,v}^{u}\|a_{j}^{(t+1)}-h_{i_{u},v}^{(t)}\|^{2} \nonumber \\
& \geq & m \sum \limits_{j=1}^{m}{\hat{W}_{2}^{2}(G_{j}^{(t)'},P_{n}^{j})}+\lambda \sum \limits_{j=1}^{m}{\hat{W}_{2}^{2}(G_{j}^{(t)'},H_{i_{j}}^{(t)})} \nonumber \\
& \geq & m \sum \limits_{j=1}^{m}{\hat{W}_{2}^{2}(G_{j}^{(t)'},P_{n}^{j})}+\lambda \sum \limits_{j=1}^{m}{d_{\hat{W}_{2}}^{2}(G_{j}^{(t)'},\Hbold^{(t)})} \nonumber \\
& = & mf(\vec{G'}^{(t)},\Hbold^{(t)}), \nonumber
\end{eqnarray}
where $\vec{G'}^{(t)}=(G_{1}^{(t)'},\ldots,G_{m}^{(t)'})$, $G_{j}^{(t)'}$ are formed by replacing the atoms of $G_{j}^{(t)}$ by the 
elements of $S_{K}^{(t+1)}$, noting that $\text{supp}(G_{j}^{(t)'}) \subseteq  \mathcal{S}_{K}^{(t+1)}$ as $1 \leq j \leq m$, and 
the second inequality comes directly from the definition of Wasserstein distance. Hence, we obtain
\begin{eqnarray}
f(\vec{G}^{(t)},\Hbold^{(t)}) > f(\vec{G'}^{(t)},\Hbold^{(t)}). \label{eqn:theorem_constraint_Wasserstein_means_first}
\end{eqnarray}
From the definition of $G_{j}^{(t+1)}$ as $1 \leq j \leq m$, we get
\begin{eqnarray}
\sum \limits_{j=1}^{m}{d_{W_{2}}^{2}(G_{j}^{(t+1)},\Hbold^{(t)})} \leq \sum \limits_{j=1}^{m}{d_{W_{2}}^{2}(G_{j}^{(t)'},\Hbold^{(t)})}. \nonumber
\end{eqnarray}
Thus, it leads to 
\begin{eqnarray}
f(\vec{G'}^{(t)},\Hbold^{(t)}) \geq f(\vec{G}^{(t+1)},\Hbold^{(t)}). \label{eqn:theorem_constraint_Wasserstein_means_second}
  \end{eqnarray}
Finally, from the definition of $H_{1}^{(t+1)},\ldots,H_{M}^{(t+1)}$, we have
\begin{eqnarray}
f(\vec{G}^{(t+1)},\Hbold^{(t)}) \geq f(\vec{G}^{(t+1)},\Hbold^{(t+1)}). \label{eqn:theorem_constraint_Wasserstein_means_third}
\end{eqnarray}
By combining \eqref{eqn:theorem_constraint_Wasserstein_means_first}, \eqref{eqn:theorem_constraint_Wasserstein_means_second}, and \eqref{eqn:theorem_constraint_Wasserstein_means_third},  we arrive at the inequality $f(\vec{G}^{(t+1)},\Hbold^{(t+1)}) < f(\vec{G}^{(t)},\Hbold^{(t)})$ when $S_{K}^{(t+1)} \neq S_{K}^{(t)}$.

In summary, we have
\begin{eqnarray}
f(\vec{G}^{(t+1)},\Hbold^{(t+1)}) < f(\vec{G}^{(t)},\Hbold^{(t)}), \nonumber
\end{eqnarray}
unless $(S_{K}^{(t+1)}, \vec{G}^{(t+1)}, \Hbold^{(t+1)}) \equiv (S_{K}^{(t)}, \vec{G}^{(t)}, \Hbold^{(t)})$. From here, using the similar argument as that of the proof of Theorem~\ref{theorem:local_convergence_multilevel_Kmeans}, we reach the conclusion of the theorem.
\paragraph{PROOF OF THEOREM \ref{theorem:objective_consistency_multilevel_Wasserstein_means}} 
To simplify notation, writing
\begin{eqnarray}
L_{\vec{n}}=\mathop {\inf }\limits_{\substack {G_{j} \in \mathcal{O}_{k_{j}}(\Theta), \\ \Hcal \in \mathcal{E}_{M}(\mathcal{P}_{2}(\Theta))}}f_{\vec{n}}(\vec{G},\Hcal) \text{ and }
L_{0}=\mathop {\inf }\limits_{\substack {G_{j} \in \mathcal{O}_{k_{j}}(\Theta), \\ \Hcal \in \mathcal{E}_{M}(\mathcal{P}_{2}(\Theta))}}f(\vec{G},\Hcal). \nonumber
\end{eqnarray}

(i) For any $\epsilon>0$, from the definition of $L_{0}$, we can find $G_{j} \in \mathcal{O}_{k_{j}}(\Theta)$ and $\Hcal \in \mathcal{E}_{M}(\mathcal{P}(\Theta))$ such that
$f(\vec{G},\Hcal)^{1/2} \leq L_{0}^{1/2} + \epsilon.$
Therefore, we would have
\begin{eqnarray}
L_{\vec{n}}^{1/2}-L_{0}^{1/2} & \leq & L_{n}^{1/2}-f(\vec{G},\Hcal)^{1/2}+\epsilon
\leq f_{\vec{n}}(\vec{G},\Hcal)^{1/2} - f(\vec{G},\Hcal)^{1/2}+\epsilon \nonumber \\
& = & \dfrac{f_{\vec{n}}(\vec{G},\Hcal)-f(\vec{G},\Hcal)}{f_{\vec{n}}(\vec{G},\Hcal)^{1/2}+f(\vec{G},\Hcal)^{1/2}} + \epsilon
\leq \sum \limits_{j=1}^{m}\dfrac{|W_{2}^{2}(G_{j},P_{n_{j}}^{j})-W_{2}^{2}(G_{j},P^{j})|}{W_{2}(G_{j},P_{n_{j}}^{j})+W_{2}(G_{j},P^{j})}+\epsilon \nonumber \\
& \leq & \sum \limits_{j=1}^{m}{W_{2}(P_{n_{j}}^{j},P^{j})}+\epsilon. \nonumber
\end{eqnarray} 
By reversing the direction, we also obtain the inequality $L_{n}^{1/2}-L_{0}^{1/2} \geq \sum \limits_{j=1}^{m}{W_{2}(P_{n_{j}}^{j},P^{j})}-\epsilon$. Hence, $|L_{n}^{1/2}-L_{0}^{1/2}-\sum \limits_{j=1}^{m}{W_{2}(P_{n_{j}}^{j},P^{j})}| \leq \epsilon$ for any $\epsilon>0$. Since $P^{j} \in \mathcal{P}_{2}(\Theta)$ for all $1 \leq j \leq m$, we obtain that $W_{2}(P_{n_{j}}^{j},P^{j}) \to 0$ almost surely as $n_{j} \to \infty$ (see for example Theorem 6.9 in \citep{Villani-2009}). As a consequence, we obtain the conclusion of part (i).

(ii) Based on the results of~\cite{Fournier_2015}, as the probability measures $P^{j}$ are compactly supported in $\mathbb{R}^{d}$ for $1 \leq j \leq m$, we have $W_{2}(P_{n_{j}}^{j},P^{j}) = O_{P}(n_{j}^{-1/d})$ when $d \geq 5$ and $W_{2}(P_{n_{j}}^{j},P^{j}) = O_{P}(n_{j}^{-1/4})$ when $d \leq 4$ for any $1 \leq j \leq m$. Combining these results with the results of part (i), we obtain that
\begin{align*}
	|L_{\vec{n}}^{1/2} - L_{0}^{1/2}| =\begin{cases} O_{P}(m \cdot n_{\vee}^{-1/d}) \ \text{when} \ d \geq 5, \\
  O_{P}(m \cdot n_{\vee}^{-1/4}) \ \text{when} \ d \leq 4.
  \end{cases}
\end{align*}
where $n_{\vee} = \max_{1 \leq i \leq m} n_{i}$. As a consequence, we reach the conclusion of part (ii).
\paragraph{PROOF OF THEOREM \ref{theorem:convergence_measures_multilevel_Wasserstein_means}} 
For any $\epsilon>0$, we denote
\begin{eqnarray}
\mathcal{A}(\epsilon)=\biggr\{G_{i} \in \mathcal{O}_{k_{i}}(\Theta), \Hcal \in \mathcal{E}_{M}(\mathcal{P}(\Theta)):
 d(\vec{G},\Hcal,\mathcal{F}) \geq \epsilon\biggr\}. \nonumber
\end{eqnarray}
Since $\Theta$ is a compact set, we also have $\mathcal{O}_{k_{j}}(\Theta)$ and $\mathcal{E}_{M}(\mathcal{P}_{2}(\Theta))$ are compact for any $1 \leq i \leq m$. As a consequence, $\mathcal{A}(\epsilon)$ is also a compact set. For any $(\vec{G},\Hcal) \in \mathcal{A}(\epsilon)$, by the definition of $\mathcal{F}$ we would have $f(\vec{G},\Hcal) > f(\vec{G}^{0},\Hcal^{0})$ for any $(\vec{G}^{0},\Hcal^{0}) \in \mathcal{F}$. Since $\mathcal{A}(\epsilon)$ is compact, it leads to
\begin{eqnarray}
\inf \limits_{(\vec{G},\Hcal) \in A(\epsilon)}{f(\vec{G},\Hcal)} > f(\vec{G}^{0},\Hcal^{0}), \nonumber
\end{eqnarray}
for any $(\vec{G}^{0},\Hcal^{0}) \in \mathcal{F}$. From the formulation of $f_{\vec{n}}$ as in the proof of Theorem \ref{theorem:objective_consistency_multilevel_Wasserstein_means}, 
we can verify that $\lim \limits_{\vec{n} \to \infty} f_{\vec{n}}(\widehat{\vec{G}}^{\vec{n}},\widehat{\Hcal}^{\vec{n}}) = \lim \limits_{\vec{n} \to \infty} f(\widehat{\vec{G}}^{\vec{n}},\widehat{\Hcal}^{\vec{n}})$ almost surely as 
$\vec{n} \to \infty$. Combining this result with that of Theorem \ref{theorem:objective_consistency_multilevel_Wasserstein_means}, 
we obtain $f(\widehat{\vec{G}}^{\vec{n}},\widehat{\Hcal}^{\vec{n}}) \to f(\vec{G}^{0},\Hcal^{0})$ as $\vec{n} \to \infty$ for any $(\vec{G}^{0},\Hcal^{0}) \in \mathcal{F}$. Therefore, for any $\epsilon>0$, as $\vec{n}$ is large enough, we have $d(\vec{\widehat{G}}^{\vec{n}},\widehat{\Hcal}^{\vec{n}},\mathcal{F}) < \epsilon$. As a consequence, we achieve the conclusion regarding the consistency of the mixing measures.

\section{Data generation processes in the simulation studies}
\label{sec:appendix_d}
In this appendix, we offer details on the data generation processes utilized in the simulation studies presented in Section \ref{Section:data_analysis} in the main text. The notions of $m, n, d, M$ are given in the main text. Let $K_i$ be the number of supporting atoms of $H_i$ and $k_{j}$ the number of atoms of $G_{j}$. For any $d \geq 1$, we denote $\vec{1}_{d}$ to be d dimensional vector with all components to be 1. Furthermore, $\mathcal{I}_{d}$ is an identity matrix with d dimensions. 

\paragraph{Comparison metric (Wasserstein distance to truth)}
\begin{eqnarray}
\text{W}:= \frac{1}{m}\sum_{j=1}^m W_2(\hat G_j, G_j) + d_{\mathcal{M}}(\hat {\Hbold}, \Hbold), \nonumber
\end{eqnarray}
where $\hat{\Hbold} := \{\hat H_1,\ldots,\hat H_M\}$, $\Hbold := \{H_1,\ldots,H_M\}$ and $d_{\mathcal{M}}(\hat H, H)$ is a minimum-matching distance \citep{tang2014understanding, Nguyen-2015}:
$$d_{\mathcal{M}}(\hat{\Hbold}, \Hbold) := \max\{\overline{d}(\hat{\Hbold}, \Hbold), \overline{d}(\Hbold, \hat{\Hbold})\},$$
where
$$\overline{d}(\hat{\Hbold}, \Hbold) := \max\limits_{1 \leq i \leq M}\,\,\min\limits_{1 \leq j \leq M} \,\, W_2(H_{i},\hat H_{j}).$$
\paragraph{Multilevel Wasserstein means setting}
The global clusters are generated as follows:
\begin{itemize}
\item Means for atoms $\mu_i := 5(i-1)$  for all $i= 1,\ldots, M$.
\item Atoms of $H_i$: $\phi_{ij} \thicksim \mathcal{N}(\mu_i \vec{1}_d, \mathcal{I}_d)$ for all $j=1,\ldots, K_i$.
\item Weights of atoms: $\pi_i \thicksim \text{Dirichlet}(\vec{1}_{K_i})$. 
\item Let $H_i := \sum_{j=1}^{K_i} \pi_{ij}\delta_{\phi_{ij}}$.
\end{itemize}
For each group $j=1,\ldots,m$, generate local measures and data as follows:
\begin{itemize}
\item Pick cluster label $z_j \thicksim \text{Uniform}( \{1,\ldots, M\})$.
\item Mean for atoms: $\tau_{ji} \thicksim H_{z_j}$ for all $i=1,\ldots, k_j$.
\item Atoms of $G_j$: $\theta_{ji} \thicksim \mathcal{N}(\tau_{ji},\mathcal{I}_d)$ for all $i=1,\ldots, k_j$.
\item Weights of atoms $p_j \thicksim \text{Dirichlet}(\vec{1}_{k_j})$.
\item Let $G_j := \sum_{i=1}^{k_j} p_{ji}\delta_{\theta_{ji}}$.
\item Data means $\mu_i \thicksim G_j$ for all $i=1,\ldots, n_j$.
\item Observations $X_{j,i} \thicksim \mathcal{N}(\mu_i,\mathcal{I}_d)$.
\end{itemize}
For the case of non-constrained variances, the variance to generate atoms $\theta_{ji}$ of $G_{j}$ is set to be proportional to global cluster label $z_{j}$ assigned to $G_{j}$.\\

\noindent
\textbf{Multilevel Wasserstein means with sharing setting} The global clusters are generated as follows:
\begin{itemize}
\item Means for atoms $\mu_i := 5(i-1)$ for all $i=1,\ldots, M$.
\item Atoms of $H_i: \phi_{ij} \thicksim \mathcal{N}(\mu_i \vec{1}_d, \mathcal{I}_d)$ for all $j=1,\ldots, K_i$.
\item Weights of atoms $\pi_i \thicksim \text{Dirichlet}(\vec{1}_{K_i})$.
\item Let $H_i := \sum_{j=1}^{K_i} \pi_{ij}\delta_{\phi_{ij}}$.
\end{itemize}
For each shared atom $k=1,\ldots,K$:
\begin{itemize}
\item Pick cluster label $z_k \thicksim \text{Uniform}( \{1,\ldots ,M \})$.
\item Mean for atoms: $\tau_k \thicksim H_{z_k}$.
\item Atoms of $S_K$: $\theta_k \thicksim \mathcal{N}(\tau_k,\mathcal{I}_d)$.
\end{itemize}
For each group $j=1,\ldots,m$ generate local measures and data as follows:
\begin{itemize}
\item Pick cluster label $\tilde z_j \thicksim \text{Uniform}( \{1,\ldots ,M \})$.
\item Select shared atoms $s_j = \{k:z_k=\tilde z_j\}$.
\item Weights of atoms $p_{s_j} \thicksim \text{Dirichlet}(\vec{1}_{|s_j|})$;
\quad
$G_j := \sum_{i \in s_j} p_{i}\delta_{\theta_{i}}$.
\item Data means $\mu_i \thicksim G_j$ for all $i=1,\ldots,n_j$.
\item Observations $X_{j,i} \thicksim \mathcal{N}(\mu_i,\mathcal{I}_d)$.
\end{itemize}
For the case of non-constrained variances, the variance to generate atoms $\theta_{i}$ of $G_{j}$ where $i \in s_{j}$ is set to be proportional to global cluster label $\tilde z_{j}$ assigned to $G_{j}$.\\

\noindent
\textbf{Three-stage K-means}
First, we estimate $G_j$ for each group $1 \leq j \leq m$ by using K-means algorithm with $k_{j}$ clusters. Then, we cluster labels using the K-means algorithm with $M$ clusters based on the collection of all atoms of $G_j$'s. Finally, we estimate the atoms of each $H_i$ via K-means algorithm with exactly $L$ clusters for each group of local atoms. Here, $L$ is some given threshold being used in Algorithm \ref{alg:multilevels_Wasserstein_means} in 
Section \ref{Section:multilevel_kmeans} in the main text to speed up the computation (see final remark regarding Algorithm \ref{alg:multilevels_Wasserstein_means} in Section \ref{Section:multilevel_kmeans}).  The three-stage K-means algorithm is summarized in Algorithm \ref{alg:three_stages_K_means}.
\setcounter{algorithm}{4}
\begin{algorithm}
   \caption{Three-stage K-means}
   \label{alg:three_stages_K_means}
\begin{algorithmic}
   \STATE {\bfseries Input:} Data $X_{j,i}$, $k_{j}$, $M$, $L$.
   \STATE {\bfseries Output:} local measures $G_{j}$ and global elements $H_{i}$ of $\Hbold$.
   \STATE {\emph{Stage 1}}
   \FOR{$j=1$ {\bfseries to} $m$}
   \STATE $G_{j} \leftarrow$ $k_{j}$ clusters of group j with K-means (atoms as centroids and weights as label frequencies).
   \ENDFOR
   	\STATE $\mathcal{C} \leftarrow$ collection of all atoms of $G_{j}$.
   \STATE {\emph{Stage 2}}
   	\STATE $\left\{D_{1},\ldots,D_{M}\right\} \leftarrow$ $M$ clusters from K-means on $\mathcal{C}$.
   	\STATE {\emph{Stage 3}}
   \FOR{$i=1$ {\bfseries to} $M$}
   \STATE $H_i \leftarrow$ $L$ clusters of $D_i$ with K-means (atoms as centroids and weights as label frequencies).
   \ENDFOR 
\end{algorithmic}
\end{algorithm}

\comment{\noindent
\textbf{Three-stage K-means with context}}

\section{Computational aspects of Wasserstein barycenter under $W_{1}$ metric}
\label{sec:appendix_e}
In this appendix, we provide a fast and efficient algorithm to compute the Wasserstein barycenter under $W_{1}$ metric. In particular, we focus on the setup when $Q_{1},\ldots, Q_{N} \in \mathcal{P}_{1}(\Theta)$ for $N \geq 1$ are finite discrete measures and we would like to determine the local Wasserstein barycenter of \eqref{eqn:Wasserstein_barycenter_first_order} within the space $\mathcal{O}_{k}(\Theta)$ for some given $k \geq 1$. 
\paragraph{Weighted geometric median:} Let $X_{1},\ldots,X_{m} \in \mathbb{R}^{d}$ be $m$ distinct points and $\eta_{1},\ldots,\eta_{m}$ be $m$ positive numbers. The weighted geometric median $X^{*} \in \mathbb{R}^{d}$ is the optimal solution of the following convex optimization problem
\begin{eqnarray}
\min \limits_{X \in \mathbb{R}^{d}}{\sum \limits_{i=1}^{m}{\eta_{i}\|X_{i}-X\|}}. \nonumber
\end{eqnarray}
To the best of our knowledge, no explicit formula for $X^{*}$ is available, therefore, we will utilize an iterative procedure to calculate an approximation for $X^{*}$. The most common approach for such procedure is Weiszfeld's algorithm \citep{Weiszfeld-1937}; however, this approach has been shown to be unstable when the update is identical to one of the given points $X_{i}$ for some $1 \leq i \leq m$. To account for this instability of Weiszfeld's algorithm,~\cite{Vardi-2000} introduces a solution for the setting when the update falls to the set of given points. In particular, their iterative algorithm can be summarized as follows
\begin{eqnarray}
X^{(i+1)} = \biggr(1 - \dfrac{\eta(X^{(i)})}{r(X^{(i)})}\biggr)^{+}\widetilde{T}(X^{(i)})+\min \left\{1,\dfrac{\eta(X^{(i)})}{r(X^{(i)})}\right\} X^{(i)}, \nonumber
\end{eqnarray}
where $\eta(x) = \begin{cases} \eta_{k} \ \text{if} \ x = X_{k}, \ k=1,\ldots,m \\ 0 \ \text{otherwise} \end{cases}$, $r(x) = \|\widetilde{R}(x)\|$ with $\widetilde{R}(x) = \sum \limits_{X_{i} \neq x}{\eta_{i}\dfrac{X_{i}-x}{\|X_{i}-x\|}}$, and $\widetilde{T}(x) = \left\{\sum \limits_{X_{i} \neq x} \dfrac{\eta_{i}}{\|X_{i}-x\|}\right\}^{-1}\dfrac{\eta_{i}X_{i}}{\|X_{i}-x\|}$ for all $x \in \mathbb{R}^{d}$. Here, we take the convention that $0/0=0$. For the convenience of argument later, we refer to this algorithm as the VZ algorithm. As being shown in \citep{Vardi-2000}, the VZ algorithm converges quickly to the global minimum of weighted geometric median problem. Due to its simplicity and efficiency in terms of computation, we will use the VZ algorithm for the updates of Wasserstein barycenter under $W_{1}$ metric.
\paragraph{Wasserstein barycenter under the entropic version of $W_{1}$ distance:}
For the simplicity of the paper, we only focus on determining Wasserstein barycenter under the entropic $W_{1}$ over the set of discrete probability measures with at most $k \geq 1$ components, i.e., we develop an efficient algorithm to estimate the optimal solution of the following optimization problem
\begin{eqnarray}
\overline{Q}_{N,\lambda}=\mathop {\arg \min}\limits_{Q \in \Ocal_{k}(\Theta)}{\sum \limits_{i=1}^{N}{\lambda_{i} \hat{W}_{1}(Q,Q_{i})}}, \label{eqn:Wasserstein_barycenter_first_order_copy}
\end{eqnarray} 
where $Q_{i} \in \Ocal_{k_{i}}(\Theta)$ for given $k_{i} \geq 1$ as $1 \leq i \leq N$ and $\lambda \in \Delta_{N}$ denotes weights associated with $Q_{1},\ldots, Q_{N}$.
\paragraph{Algorithm for Wasserstein barycenter under the entropic version of $W_{1}$ distance:}
The algorithm for determining Wasserstein barycenter of equation \eqref{eqn:Wasserstein_barycenter_first_order_copy} will follow those in \citep{Cuturi-2014} with the only modification regarding updating the atoms of $\overline{Q}_{N,\lambda}$ in terms of VZ algorithm for the geometric median. We summarize that algorithm in Algorithm~\ref{alg:Wasserstein_barycenter_W1_metric}.
\setcounter{algorithm}{5}
\begin{algorithm}
   \caption{Wasserstein barycenter under the entropic version of $W_{1}$ metric}
   \label{alg:Wasserstein_barycenter_W1_metric}
\begin{algorithmic}
   \STATE {\bfseries Input:} Atoms $Y_{i} \in \mathbb{R}^{d \times k_{i}}$ of $Q_{i}$, weights $b_{i}$ of $Q_{i}$, and $\lambda \in \Delta_{N}$.
   \STATE {\bfseries Output:} Atoms $X \in \mathbb{R}^{d \times k}$ and weights $a$ of $\overline{Q}_{N,\lambda}$.
   \STATE Initialize atoms $X^{(0)}$, weights $a^{(0)}$ of $\overline{Q}_{N,\lambda}^{(0)}$, and $t=0$.
    \WHILE{$X^{(t)}$ and $a^{(t)}$ have not converged}
   \STATE Update $a^{(t)}$ using Algorithm 1 in \citep{Cuturi-2014}.
   \FOR{$i=1$ {\bfseries to} $N$}
   \STATE $T^{i} \leftarrow$ optimal coupling of $\overline{Q}_{N,\lambda}^{(t)}$, $Q_{i}$.
   \ENDFOR
   \FOR{$i=1$ {\bfseries to} $k$}
   \STATE $\widetilde{X}_{i}^{(t)} \leftarrow \mathop {\arg \min} \limits_{X_{i} \in \mathbb{R}^{d}} \sum \limits_{j=1}^{N} \lambda_{j} \sum \limits_{u=1}^{k_{j}} T^{j}_{i,u}\|X_{i} - Y_{j,u}\|$. 
   \ENDFOR
   \STATE $\widetilde{X}^{(t)} \leftarrow [\widetilde{X}_{1}^{(t)},\ldots,\widetilde{X}_{k}^{(t)}]$.
   \STATE $X^{(t)} \leftarrow (1-\theta)X^{(t)}+\theta\widetilde{X}^{(t)}$ where $\theta \in [0,1]$ is a line-search or preset value.
   \STATE $t \leftarrow t+1$.
   \ENDWHILE
\end{algorithmic}
\end{algorithm}
\section{Consistency of estimators from Multilevel Wasserstein Geometric Median (MWGM) method}
\label{sec:appendix_f}
In this appendix, we provide consistency of the objective function and the estimators of the MWGM method. To simplify the presentation, we also fix $m$ and assume that $P^{j}$ is the true distribution of data $X_{j,i}$ for $j = 1,\ldots,m$. Similar to Section~\ref{Section:consistency_multilevel_Kmeans}, we denote $\vec{G}=(G_{1},\ldots,G_{m})$ and $\vec{n}=(n_{1},\ldots,n_{m})$. We say that $\vec{n} \to \infty$ if $n_{j} \to \infty$ for $j=1,\ldots, m$. Define the following functions associated with the MWGM method and its population version:
\begin{eqnarray}
g_{\vec{n}}(\vec{G},\Hcal)=\mathop {\sum }\limits_{j=1}^{m}{W_{1}(G_{j},P_{n_{j}}^{j})}+ \lambda W_{1}(\Hcal,\dfrac{1}{m}\mathop {\sum }\limits_{j=1}^{m}{\delta_{G_{j}}}), \nonumber \\
g(\vec{G},\Hcal)=\mathop {\sum }\limits_{j=1}^{m}{W_{1}(G_{j},P^{j})}+\lambda W_{1}(\Hcal,\dfrac{1}{m}\mathop {\sum }\limits_{j=1}^{m}{\delta_{G_{j}}}), \nonumber
\end{eqnarray}
where $G_{j} \in \mathcal{O}_{k_{j}}(\Theta)$, $\Hcal \in \mathcal{E}_{M}(\mathcal{P}_{1}(\Theta))$ for $1 \leq j \leq m$. 
We first establish the first consistency property of the MWGM method.
\begin{thm} \label{theorem:objective_consistency_multilevel_Wasserstein_median} 
Assume that  $P^{j} \in \mathcal{P}_{1}(\Theta)$ for $1 \leq j \leq m$. Then, as $\vec{n} \to \infty$
\begin{eqnarray}
\mathop {\inf }\limits_{\substack {G_{j} \in \mathcal{O}_{k_{j}}(\Theta), \\ \Hcal \in \mathcal{E}_{M}(\mathcal{P}_{1}(\Theta))}}g_{\vec{n}}(\vec{G},\Hcal) - \mathop {\inf }\limits_{\substack {G_{j} \in \mathcal{O}_{k_{j}}(\Theta), \\ \Hcal \in \mathcal{E}_{M}(\mathcal{P}_{1}(\Theta))}}g(\vec{G},\Hcal) \rightarrow 0 \nonumber
\end{eqnarray}
almost surely.
\end{thm}
\begin{proof}
The proof of Theorem~\ref{theorem:objective_consistency_multilevel_Wasserstein_median} follows the same proof argument as that of Theorem~\ref{theorem:objective_consistency_multilevel_Wasserstein_means}. Here, we provide the proof for the completeness. We denote
\begin{eqnarray}
L_{\vec{n}}=\mathop {\inf }\limits_{\substack {G_{j} \in \mathcal{O}_{k_{j}}(\Theta), \\ \Hcal \in \mathcal{E}_{M}(\mathcal{P}_{1}(\Theta))}}g_{\vec{n}}(\vec{G},\Hcal) \text{ and }
L_{0}=\mathop {\inf }\limits_{\substack {G_{j} \in \mathcal{O}_{k_{j}}(\Theta), \\ \Hcal \in \mathcal{E}_{M}(\mathcal{P}_{1}(\Theta))}}g(\vec{G},\Hcal). \nonumber
\end{eqnarray}
For any $\epsilon>0$, from the definition of $L_{0}$, we can find $G_{j} \in \mathcal{O}_{k_{j}}(\Theta)$ and $\Hcal \in \mathcal{E}_{M}(\mathcal{P}(\Theta))$ such that
$g(\vec{G},\Hcal) \leq L_{0} + \epsilon$. An application of triangle inequality leads to
\begin{eqnarray*}
L_{\vec{n}} - L_{0} & \leq & L_{n} - g(\vec{G},\Hcal) + \epsilon
\leq g_{\vec{n}}(\vec{G},\Hcal) - g(\vec{G},\Hcal) + \epsilon \nonumber \\
& \leq & \sum \limits_{j=1}^{m}{W_{1}(P_{n_{j}}^{j},P^{j})}+\epsilon.
\end{eqnarray*} 
By reversing the direction, we also obtain the inequality $L_{n} - L_{0} \geq \sum \limits_{j=1}^{m}{W_{1}(P_{n_{j}}^{j},P^{j})}-\epsilon$. Putting the two inequalities together, we have
\begin{align*}
|L_{n} - L_{0} - \sum \limits_{j=1}^{m}{W_{1}(P_{n_{j}}^{j},P^{j})}| \leq \epsilon
\end{align*}
for any $\epsilon>0$. According to the hypothesis, $P^{j} \in \mathcal{P}_{1}(\Theta)$ for all $1 \leq j \leq m$. Therefore, we obtain that $W_{1}(P_{n_{j}}^{j},P^{j}) \to 0$ almost surely as $n_{j} \to \infty$. As a consequence, we obtain the conclusion of the theorem.
\end{proof}
Our next result establishes that the consistency of the estimators from the MWGM method to those in the population version of MWGM.  To ease the presentation, we assume that for each $\vec{n}$ there is an optimal solution $
(\widehat{G}_{1}^{n_{1}},\ldots,\widehat{G}_{m}^{n_{m}},\widehat{\Hcal}^{\vec{n}})$ (or in 
short $(\vec{\widehat{G}}^{\vec{n}},\Hcal^{\vec{n}})$) of the MWGM objective function 
\eqref{eqn:robust_multilevel_Wasserstein_median_typeone}. Furthermore, we can find
optimal solution minimizing $f(\vec{G},\Hcal)$ over $G_{j} \in \mathcal{O}_{k_{j}}(\Theta)$ and $\Hcal \in 
\mathcal{E}_{M}(\mathcal{P}_{1}(\Theta))$. We denote $\mathcal{F}$ the collection 
of these optimal solutions. For any $G_{j} \in \mathcal{O}_{k_{j}}(\Theta)$ and $\Hcal \in 
\mathcal{E}_{M}(\mathcal{P}_{1}(\Theta))$, we define
\vspace{-6pt}
\begin{eqnarray}
\bar{d}(\vec{G},\Hcal,\mathcal{F})=\inf \limits_{(\vec{G}^{0}, \Hcal^{0}) \in \mathcal{F}}\sum \limits_{j=1}^{m}{W_{1}(G_{j},G_{j}^{0})} \nonumber
+W_{1}(\Hcal,\Hcal^{0}). \nonumber
\end{eqnarray}
Given the above assumptions and the consistency of the objective function of the MWGM method, we have the following result regarding the convergence of $(\widehat{\vec{G}}^{\vec{n}},\Hcal^{\vec{n}})$.
\begin{thm} \label{theorem:convergence_measures_multilevel_Wasserstein_median}
Assume that $\Theta$ is bounded and $P^{j} \in \mathcal{P}_{1}(\Theta)$ for all $1 \leq j 
\leq m$. Then, we have $\bar{d}(\vec{\widehat{G}}^{\vec{n}},\widehat{\Hcal}
^{\vec{n}},\mathcal{F}) \to 0$ as $\vec{n} \to \infty$ almost surely.
\end{thm}

The proof of Theorem~\ref{theorem:convergence_measures_multilevel_Wasserstein_median} is a direct combination of the proof argument of Theorem~\ref{theorem:convergence_measures_multilevel_Wasserstein_means} and the consistency of objective function of MWGM method in Theorem~\ref{theorem:objective_consistency_multilevel_Wasserstein_median}; therefore, it is omitted.

\section{Algorithms for computing the entropic Wasserstein barycenters} \label{sec:appendix_g}

\begin{algorithm}[H]
   \caption{Smoothed Primal $T^*_\gamma$ and Dual $b^*_\gamma$ Optima}
   \label{alg:T_and_b}
\begin{algorithmic}
   \STATE {\bfseries Input:} $M=M_{XY}$,$\gamma,a,b$.
   \STATE {\bfseries Output:} $T^*_\gamma$ and $b^*_\gamma$.
   \STATE $K=\exp(-\gamma M)$
   \STATE $\tilde{K}=\text{diag}(b^{-1})K$
    \STATE Initialize $u=\text{ones}(k,1)/k$
    \WHILE{$u$ changed}
  \STATE  $u=1./(\tilde{K}(b./(K^\top u)))$
   \ENDWHILE
   \STATE $v=b./(K^\top u)$.
     \STATE $b^*_\gamma=\frac{1}{\gamma}\log u+\frac{\log{u}^\top\mathbf{1}_k}{\gamma k}\mathbf{1}_k$.
       \STATE $T^*_\gamma=\text{diag}(u)K\text{diag}(v)$.
    \RETURN $T^*_\gamma, b^*_\gamma$
\end{algorithmic}
\end{algorithm}
In this appendix, we include three algorithms for computing entropic Wasserstein barycenter in \citep{Cuturi-2014} to make our manuscript self-contained. Before describing the details of the algorithms, we summarize the notations as follows. Let $X_{1},\ldots,X_{n} \in \mathbb{R}^{d}$ (resp., $Y_{1},\ldots,Y_{k} \in \mathbb{R}^{d}$) be $n$ (resp., $k$) atom points of discrete measure $P=\sum \limits_{i=1}^{n}a_i\delta_{X_i}$ (resp., $Q=\sum \limits_{i=1}^{k}b_i\delta_{Y_i}$) where $a=(a_1,\ldots,a_n)\in \Delta_n$  (resp., $b=(b_1,\ldots,b_k)\in \Delta_k$). The matrix $M_{XY}$ of pairwise distances between elements of $X_i$ and $Y_j$ raised to the power $p$, i.e. $M_{XY} \in \mathbb{R}^{n \times k}$. The transportation polytope is denoted as $U(a, b)=\{\mathbb{R}_{+}^{n \times k}\mid T\mathbf{1}_k=a,T^\top\mathbf{1}_n=b\}.$ The Wasserstein distance raised to power $2$ now reads
\vspace{-2pt}
\begin{eqnarray}
W^2_{2}(P,Q)=\min_{T\in U(a, b)} \trace{(T^\top M_{XY})}, \nonumber
\end{eqnarray} which  has its entropic relaxed formula
\vspace{-2pt}
\begin{eqnarray}
\hat{W}^2_{2}(P,Q)=\min_{T\in U(a, b)} \trace{(T^\top M_{XY})}+\gamma\sum\limits_{i,j=1}^{n,k} T_{ij}\log T_{ij}. \label{eqn:relaxed_discrete_W}
\end{eqnarray}
The Wasserstein barycenter up to $k$ support points of $P_j$, i.e. $P_j=\sum \limits_{i=1}^{m} a_{ji}\delta_{X_{ji}}$, for $1\leq j\leq N$ with weights of $\omega=(\omega_1,\ldots,\omega_N)\in \Delta_N$ is the minimizer of the following problem
\vspace{-6pt}
\begin{eqnarray}
Q=\min_{b,\boldsymbol{Y}}\sum \limits_{j=1}^{N}\omega_j\trace{(T^\top M_{X_{j}Y})}, \label{eqn:discrete_barycenter}
\end{eqnarray} where $b\in \Delta_k$ and $\boldsymbol{Y}={Y_1,\ldots, Y_k}$.

Algorithm \ref{alg:T_and_b} is designed to compute the  optimal transport plan $T^*_\gamma$ and dual optima $b^*_\gamma$ (aka sub-gradient) of relaxed Wasserstein distance with respect to $b$ in equation~\eqref{eqn:relaxed_discrete_W}. Algorithm \ref{alg:fix_support_WB_Cuturi} describes routines for computing the weights of the barycenter in \eqref{eqn:discrete_barycenter}. The procedure for computing free support barycenter in \eqref{eqn:discrete_barycenter} is summarized in Algorithm \ref{alg:free_support_WB_Cuturi}.

\begin{algorithm}
   \caption{Fix-support Wasserstein barycenter}
   \label{alg:fix_support_WB_Cuturi}
\begin{algorithmic}
   \STATE {\bfseries Input:} Atoms $X_j \in \mathbb{R}^{d \times n_j}$ and weights $a_{j}$ of $P_{j}$, cost matrices $M_j=M_{X_{j}Y}$ and $\omega \in \Delta_{N}$, $\gamma, t_0>0$.
   \STATE {\bfseries Output:} Weights $b$ of $Q_{N}$.
   \STATE Initialize  $\tilde{b}$ and set $\hat{b}=\tilde{b}$.
   \STATE Set $t=0$.
    \WHILE{$\hat{b}$ has not converged}
   \STATE $\beta=(t+1)/2$, $b=(1-\beta^{-1})\hat{b}+\beta^{-1}\tilde{b}$.
   \STATE Compute sub-gradient $\alpha_j$s using Algorithm \ref{alg:T_and_b} using $M_{j},b$ and $a_j$
   \STATE $\alpha=\sum\limits_{j=1}^{N}\omega_j \alpha_j$
    \STATE $\tilde{b}=\tilde{b}*\exp(-t_0\beta\alpha)$,  $\tilde{b}=\tilde{b}/\tilde{b}^\top\mathbf{1}_k$
      \STATE $\hat{b}=(1-\beta^{-1})\hat{b}+\beta^{-1}\tilde{b}$
   \ENDWHILE
   \RETURN $\hat{b}$
\end{algorithmic}
\end{algorithm}

\begin{algorithm}[t]
   \caption{Free-support Wasserstein barycenter}
   \label{alg:free_support_WB_Cuturi}
\begin{algorithmic}
   \STATE {\bfseries Input:} Atoms $X_j \in \mathbb{R}^{d \times n_j}$ and weights $a_{j}$ of $P_{j}$, cost matrices $M_j=M_{X_{j}Y}$ and $\omega \in \Delta_{N}$, $\gamma, t_0>0$, $k$ -- the number of support points of barycenter.
   \STATE {\bfseries Output:} Atoms $Y\in \mathbb{R}^{d \times k}$ and weights $b$ of $Q_{N}$.
   \STATE Initialize  $Y$ and $b$.
   \STATE Set $t=0$.
    \WHILE{$Y$ and $b$ have not converged}
    \STATE Computing $b$ using Algorithm \ref{alg:fix_support_WB_Cuturi}
    \FOR{$j=1$ {\bfseries to} $N$}
   \STATE Computing $T_j$ using Algorithm \ref{alg:T_and_b}
   \ENDFOR
   \STATE Setting $\theta \in [0,1]$ with line search or a preset value.
   \STATE $X=(1-\theta)X+\theta(\sum\limits_{j=1}^{N}\omega_j X_j T^{\top}_j)\text{diag}(b^{-1})$
   \ENDWHILE
   \RETURN $X,b$
\end{algorithmic}
\end{algorithm}

\vskip 0.2in
\bibliographystyle{unsrtnat}
\bibliography{Nhat,NPB,Nguyen,MY_ref}
\end{document}